\documentclass[11pt, onecolumn]{article}
\usepackage[top=1in, bottom=1in, left=1.25in, right=1.25in]{geometry}

\usepackage{natbib}
\setcitestyle{authoryear,open={(},close={)}}

\usepackage{bm}
\usepackage[cmex10]{amsmath}

\usepackage{graphicx}
\usepackage{amsfonts}
\usepackage[cmex10]{amsmath}
\usepackage{amssymb}
\usepackage{amsthm}
\usepackage{mdwmath}
\usepackage{algorithm}
\usepackage{algorithmic}
\usepackage{subfigure}
\usepackage{subeqnarray}
\usepackage{cases}
\usepackage{tabularx}
\usepackage{multirow}
\usepackage{enumerate}
\usepackage{subfigure}
\usepackage{url}
\usepackage{flushend}
\usepackage{tikz}
\usepackage{array}
\usepackage{booktabs}

\newcommand{\tabincell}[2]{\begin{tabular}{@{}#1@{}}#2\end{tabular}}

\usepackage{color,soul}

\theoremstyle{definition}
\newtheorem{definition}{Definition}

\theoremstyle{definition}

\theoremstyle{definition}
\newtheorem{lemma}{Lemma}

\theoremstyle{definition}
\newtheorem{theorem}{Theorem}

\theoremstyle{definition}

\theoremstyle{remark}
\newtheorem{remark}{Remark}

\newcommand{\set}{\mathcal}

\def\x{{\bm x}}

\begin{document}

\title{Compressed Subspace Learning Based on Canonical Angle Preserving Property}

\author{Yuchen~Jiao, Gen~Li, and~Yuantao~Gu%
\thanks{ 
The authors are with Department of Electronic Engineering, Tsinghua University, Beijing 100084, China. 
The corresponding author of this paper is Y. Gu (gyt@tsinghua.edu.cn).}
}
\date{Manuscript submitted July 14, 2019.}

\maketitle

\begin{abstract}
Union of Subspaces (UoS) is a popular model to describe the underlying low-dimensional structure of data.
The fine details of UoS structure can be described
in terms of canonical angles (also known as principal angles) between subspaces,
which is a well-known characterization for relative subspace positions.
In this paper, we prove that random projection with the so-called Johnson-Lindenstrauss (JL) property
 approximately preserves canonical angles between subspaces with overwhelming probability.
This result indicates that random projection approximately preserves the UoS structure.
Inspired by this result, we propose a framework of Compressed Subspace Learning (CSL),
which enables to extract useful information from the UoS structure of data in a greatly reduced dimension.
We demonstrate the effectiveness of CSL in various subspace-related tasks
such as subspace visualization, active subspace detection, and subspace clustering.

\end{abstract}

\section{Introduction}


Many data analysis tasks in machine learning and data mining deal with
real-world data that are presented in high-dimensional spaces,
which often brings prohibitively high computational complexity.
Attempting to resolve this problem led data scientists to the
discovery that many high-dimensional real-world data sets
possess some low-dimensional structure that makes them easier to handle.
Various models have been proposed to describe such structures,
among which the Union of Subspaces (UoS) model is a popular one.
It assumes that in a data set with high ambient dimension,
the data points actually lie on a few low-dimensional linear subspaces.
This model has successfully characterized the intrinsic low-dimensional structure of many data sets,
including face images from multiple individuals,
marker trajectories from multiple rigid objectives,
hyperspectral images,
and gene expression data \citep{Elhamifar2009Sparse,Ehsan2013Sparse,Mcwilliams2014Subspace,Zhai2017Anew}.

The task of subspace learning%
\footnote{Here the subspace learning has a different meaning from it is in
 some other literatures.}
 is then to extract useful information
from UoS structure of data.
For example, subspace clustering
seeks to simultaneously segment data with the same underlying structure and
estimate the latent low-dimensional subspaces,
active subspace detection assigns category labels to newly-encountered data points
by identifying the subspace they lie in,
subspace visualization helps to discover the correlation and irregularity
in a data set, like outliers distributing in the whole space.
Many
algorithms of the above subspace learning tasks have been proposed.
The performance of these algorithms has been found closely related to
the concept of \emph{subspace structure}%
\footnote{Here \emph{subspace structure} is a rough concept describing relative subspace positions.
It may represent affinity, subspace distances, and canonical angles between subspaces.}
\citep{Noisy2016Wang,Heckel2015dimensionality,Meng2018Ageneral,Lodhi2017Detection}.

There is, however, a natural question that gets unnoticed in the design of these classical algorithms, to cite \citet{donoho2006compressed}, ``why go to so much effort to acquire all the data when most of what we get will be thrown away?''
In our case, this translates to the following:
 since UoS structure involves only a collection of low-dimensional subspaces that cost much less to describe than the original high-dimensional representation of all data points,
why do we go to so much effort to process the redundant high-dimensional representation?
This motivates us to propose in this paper the framework of Compressed Subspace Learning (CSL),
which significantly reduces the sampling and processing complexity of subspace learning
by utilizing random projection to map the original data to a space with dimension
$O(d)$, where $d$ is the maximal dimension of underlying subspaces in UoS model.
At most interesting applications, $d$ is indeed extremely low
compared with the ambient dimension of the data \citep{Noisy2016Wang}.
It is obvious that this bound cannot be improved:
a $d$-dimensional subspace cannot be embedded in a space with dimension less than $d$,
thus there is no way to preserve the UoS structure if the data is to be mapped into
a space with dimension less than $d$.

To analytically characterize the impact exerted by random projection on UoS structure,
we restrict our attention to a class of random projections with so-called
Johnson-Lindenstrauss (JL) property \citep{Foucart2017Mathematical}.
This choice is advantageous in that JL property is a strong concentration property
yet satisfied by a very wide range of random matrices,
such as Gaussian matrices, Bernoulli matrices, other sub-Gaussian matrices,
and some matrices with fascinating fast algorithms.
For such random projection we prove that
the UoS structure, described in terms of canonical angles,
is approximately preserved after being projected onto a space of dimension $O(d)$.
We call this property \emph{Canonical Angle Preserving} (CAP) property.
CAP property forms the theoretical foundation of our CSL framework.
We test the performance preserving property of our framework on several subspace-related tasks,
including subspace visualization, active subspace detection, and subspace clustering.

\subsection{Random Projection and Its Structure Preserving Property}
\label{subsec:intro-preserving-property-GP}
Among numerous dimensionality reduction methods,
linear methods are widely used in practice
for their simple geometric interpretations and
computational efficiency.
The most famous one in this category may be Principal Component Analysis (PCA),
which projects the original data onto a low-dimensional space such
that the dimensionality-reduced training data has the maximized variance.
Random projection is another famous family of linear methods,
which reduces the dimension of original data by multiplying it with a fat random matrix.
Random projection has the advantage of high computational efficiency
and being data-free.

More technically speaking,
random projection uses a randomly generated matrix
to map the original high-dimensional
data in $\mathbb{R}^N$ to a low-dimensional space $\mathbb{R}^n$, $n<N$.
It is, of course, impossible to undertake a comprehensive study on all types of random projections,
and practice indicates there are only a few random matrices
that are interesting enough to be used for random projection.
Typical examples include Gaussian matrices, Bernoulli matrices, other sub-Gaussian matrices,
partial Fourier matrices, and partial Hadamard matrices.
Though many previous works on random projection and random matrix theory
focus on sub-Gaussian matrices,
structured random matrices like partial Fourier matrices and partial Hadamard matrices are also important due to computational convenience.
In fact,
the computational complexity of random projection is $O(Nn)$ for sub-Gaussian matrices,
and $O(N\log N)$ for partial Fourier matrices and partial Hadamard matrices.
Note that the $O(Nn)$ complexity is at the same level with PCA,
while the $O(N\log N)$ complexity is close to optimal
since computing the dimensionality-reduced image of a generic $N$-dimensional vector
requires at least $O(N)$ time (to read the input).

It turns out that there is a systematic scheme to treat
most of the aforementioned random matrices
via JL property,
most notably the ones with fast algorithms \citep{Foucart2017Mathematical, xingyu}.
We will study under this scheme the distortion of subspace structure brought by random projection.

\begin{figure}[t]
  \centering
  \centerline{\includegraphics[width=\textwidth]{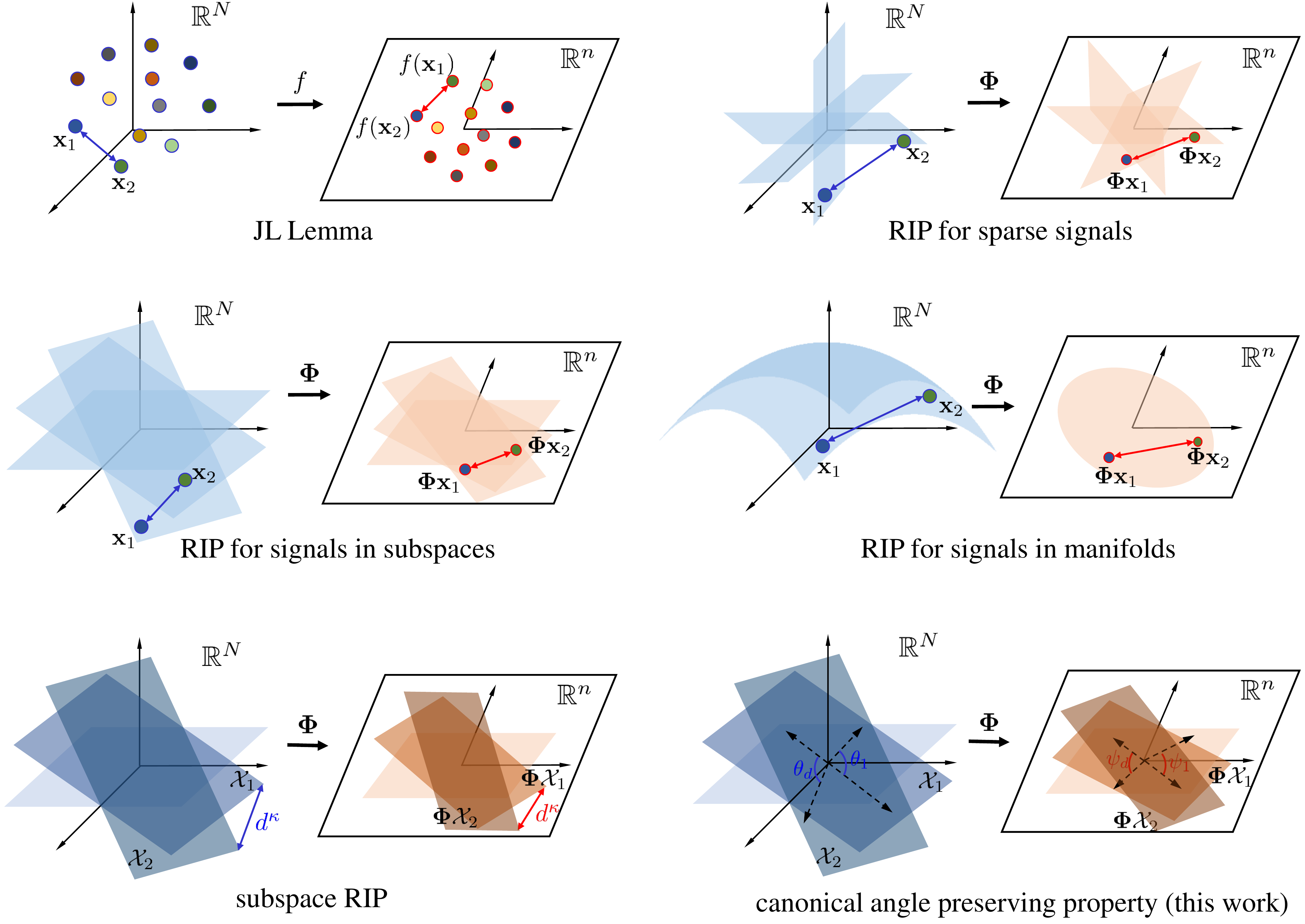}}
\caption{JL Lemma, RIP for sparse signals,
RIP for signals in subspaces,
RIP for signals in manifolds,
subspace RIP, and canonical angle preserving property.
Notation $f$, ${\bm \Phi}$ denotes the linear map and random projection
from $\mathbb{R}^N$ to $\mathbb{R}^n$, respectively, $n<N$.
Notation ${\bm x}_i$ and $\mathcal{X}_i$ denotes the point and $d$-dimensional subspace
in Euclidean space, respectively.
Notation $d^{\kappa}$ denotes the projection Frobenius-norm distance,
and $\theta_k$, $\psi_k$ denotes the $k$-th canonical angle before and after random projection, respectively, $k=1,\cdots,d$.}
\label{fig:JL}
\end{figure}

The investigation of distortion on subspace structure induced by random projection
fits into the long history of researches on structure preserving property of random projection.
Figure~\ref{fig:JL} depicts some results in this vein.
The story begins with the classical Johnson-Lindenstrauss Lemma,
which considers the structure of point sets in Euclidean space described by pairwise distance.
JL Lemma states that
for any a set $S$ consisting of $L$ points in Euclidean space $\mathbb{R}^N$,
there is a map $f:\mathbb R^N\to\mathbb R^n$
where $n=O(\varepsilon^{-2}\log L)$,
such that all pairwise Euclidean distances in $S$ are
preserved up to a factor of $(1\pm\varepsilon)$.
This result is originally proved by choosing $f$ to be Gaussian random projection \citep{Johnson1984Extensions}.
JL Lemma has now become a fundamental lemma in the theory of machine learning.
Another notion related to JL Lemma is the classical Restricted Isometry Property (RIP) for sparse signals,
which states that all sparse vectors in $\mathbb{R}^N$ with sparsity no more than $k$
can be embedded into $O(k\log(N/k))$ dimensions with the
pairwise Euclidean distances preserved up to $(1\pm\varepsilon_{2k})$
\citep{Cand2008The, Baraniuk2015A}.
It has been proved that sub-Gaussian random matrices and some sparse random matrices satisfy RIP for sparse signals
with probability $1-{\rm e}^{-O(n)}$ \citep{Candes2006Compressive, Eftekhari2015Therestricted}.
This conclusion has remarkably fascinated the researches in compressed sensing.
More researches show that sub-Gaussian random matrices are able to preserve some other low-dimensional structures, for instance, pairwise distance of data points on subspaces and manifolds \citep{Dirksen2007Dimensionality}.
These results are named as RIP for signals in subspaces and manifolds.

In the recent decade, the powerful UoS model
leads to a new point of view that the structure of many real-world data sets
is in fact the structure of a collection of subspaces where the data points reside
\citep{Eldar2009Robust, Complex2014Zhu}.
In spite of the extensive study in the literature on the distance preserving property for data points,
it was not clear whether random projection preserves the distance or more refined structure
of subspaces until the emergence of \citet{LiGenTSP} and \citet{optimal}.
In these two papers it is proved that Gaussian random projection
can approximately preserve the affinity between two subspaces.
These two papers also proved that the so-called projection Frobenius-norm distance
of subspaces are approximately preserved
and named this property \emph{subspace RIP}.
More precisely,
in \citet{optimal} it is stated that
any $L$ given subspaces with dimensions at most $d$
can be embedded by Gaussian random matrices
into $O(\max\{d,\log L,\log(1/\delta)\})$ dimensions
with probability $1-\delta$,
such that their pairwise projection Frobenius-norm distances
are preserved up to a factor of $(1\pm\varepsilon)$.

Subspace structure plays an essential role in many algorithms based on UoS model.
For example, it has been proved that subspace affinity
or canonical angles influence
the performance of subspace clustering algorithms,
including Sparse Subspace Clustering (SSC) in \citet{Ehsan2013Sparse}, thresholding-based subspace clustering (TSC) in \citet{Heckel2015robust}, and SSC via Orthogonal Matching Pursuit (SSC-OMP) in \citet{Dyer2013Greedy} and \citet{You2016Scalable}.
When applying these algorithms on dimensionality-reduced data sets,
subspace structure preserving property turns out to be a useful tool in analyzing their performance \citep{Meng2018Ageneral}.


\subsection{Canonical Angles}
\label{subsec:intro-principal-angles}

It is well-known that subspace structure is perfectly described by canonical angles,
also known as principal angles \citep{jordan1875essai,Wong1967Differential}.
These are a sequence of acute angles
that provide a complete characterization of the relative subspace positions in the following sense:
\begin{theorem}\citep{Wong1967Differential}
If the canonical angles between subspaces $\mathcal{S}_1,\mathcal{S}_2$
are identical with the canonical angles between subspaces $\mathcal{S}_1^{\prime},\mathcal{S}_2^{\prime}$,
then there exists an orthogonal transform $\bm T$ such that
${\bm T}\mathcal{S}_1=\mathcal{S}_1^{\prime}$, ${\bm T}\mathcal{S}_2=\mathcal{S}_2^{\prime}$.
\end{theorem}

It is thus obvious that any other quantity describing relative subspace positions
is a function of canonical angles,
for example, the affinity between two subspaces \citep{soltanolkotabi2012geometric},
any notion of rotation-invariant subspace distance,
including the aforementioned projection Frobenius-norm distance
and the widely-used geodesic distance \citep{YeLim},
other definitions of subspace angles, including
product angle \citep{Miao1992Onprincipal, Miao1996product},
Friedrichs angle, and Dixmier angle \citep{Deutsch1995The}.
See Appendix \ref{app:subspace-distance} for a discussion on different definitions of subspace distance, and the advantage of canonical angles over these subspace distances
in characterizing relative subspace positions.

\subsection{Contributions}

This work first studies the distortion of canonical angles
induced by random projection with JL property.
To be precise,
it is proved that for any $L$ given subspaces with dimensions at most $d$,
they can be mapped to a low-dimensional space $\mathbb{R}^n$
with each canonical angle
preserved up to $(1\pm\varepsilon)$
with probability $1-\delta$.
The requirement on dimension $n$ is given by $n=O(\varepsilon^{-2}\max\{d,\log L,\log(1/\delta)\})$.
This result indicates that each canonical angle is approximately preserved
by random projection with JL property,
and thus is called canonical angle preserving (CAP) property.
As canonical angles best characterize the relative subspace positions,
CAP property implies that subspace structure also remains almost unchanged after dimensionality reduction.
Based on CAP property,
some other important concepts on subspace structure,
such as various notions of subspace distance,
are also proven to be almost invariant.

With CAP property as the theoretical foundation,
we propose the Compressed Subspace Learning (CSL) framework,
which enables to process data
in a space with reduced dimension that is much lower than the ambient dimension
without deteriorating the performance.
We verify the effectiveness of this framework
on three concrete subspace-learning tasks, namely subspace visualization, active subspace detection, and subspace clustering.
The experiments and theoretical analyses show that
the performance of all of these three algorithms
are almost preserved.
Another observation on subspace clustering is that
applying CSL framework successfully circumvents the curse of dimensionality
for it significantly reduces the dimension of the data by JL random projection.
Considering that CAP property is independent of algorithms,
we infer that CSL is a universally effective framework
for subspace-related tasks.

\subsection{Organization and Notations}

The rest of this paper is organized as follows.
In Section~\ref{sec:preliminary}, definitions and basic properties about canonical angles and JL property are provided.
In Section~\ref{sec:principal-angle-preserving},
we precisely state our main theoretical result, i.e., CAP property,
and use it to establish a general subspace RIP.
With these theoretical results as foundation,
we formulate the CSL framework in Section~\ref{sec:framework},
and give some description.
Section~\ref{sec:proof-theorem-multi-angle} is devoted to a full proof of CAP property.
In Section~\ref{sec:app-compression},
we empirically show the effectiveness of CSL framework
on three subspace-related tasks.
The corresponding performance analysis are deferred to appendix.
Finally, in Section~\ref{sec:conclusion} we conclude the paper.

Throughout this paper,
bold upper and lower case letters are used to denote matrices and vectors, respectively.
Notation ${\bm V}^{\rm T}$ denotes the transposition of matrix ${\bm V}$.
Notation $\left\|{\bm v}\right\|$ denotes the $\ell_2$-norm of vector ${\bm v}$.
$\sigma_k({\bm V})$ denotes the $k$-th singular value of matrix ${\bm V}$ in descending order.
Letters in calligraphy denote subspaces, such as ${\set X}$, ${\set Y}$ and ${\set S}$.
The orthogonal  complement of subspace ${\set S}$ is denoted by ${\set S}^{\perp}$.
The projection of vector ${\bm v}$ onto subspace $\mathcal{S}$ is denoted as ${\bm P}_{\mathcal S}({\bm v})$.
The $(k-1)$-dimensional unit sphere is denoted by $\bm{S}^{k-1}$.

\section{Preliminary}
\label{sec:preliminary}

We now give the precise definition of canonical angles discussed in Section~\ref{subsec:intro-principal-angles}
as below.

\begin{definition}\citep{PrincipalAngles2006}\label{principal-angle}
Assume there are two subspaces ${\set X}_1$, ${\set X}_2$ with dimensions $d_1\le d_2$.
There are $d_1$ canonical angles $\theta_1$, $\cdots$, $\theta_{d_1}$ between them, which
are recursively defined as
$$
\cos{\theta_k}=\max\limits_{{\bm x}_1 \in {\set X}_1}\max\limits_{{\bm x}_2 \in {\set X}_2}\frac{{\bm x}_1^{\rm T}{\bm x}_2}{\Vert {\bm x}_1\Vert \Vert {\bm x}_2\Vert}=:\frac{{\bm x}_{1,k}^{\rm T}{\bm x}_{2,k}}{\Vert {\bm x}_{1,k}\Vert\Vert {\bm x}_{2,k}\Vert}, \quad k=1,\cdots,d_1,
$$
where the maximization is with the constraints ${\bm x}_l^{\rm T}{\bm x}_{li}=0$,
 $i=1,\cdots,k-1$, $l=1,2$.
The vectors $({\bm x}_{1,k},{\bm x}_{2,k})$ are the corresponding pairs of principal vectors, $k=1,\cdots, d_1$.
\end{definition}

Clearly $0\le\theta_1\le\cdots\le\theta_{d_1}\le \pi/2$.
We remark that canonical angles are uniquely defined, while principal vectors are not.

An alternative definition of canonical angles and principal vectors via singular values is stated as below, which is equivalent to Definition \ref{principal-angle}.

\begin{lemma}\citep{PrincipalAngles1973}\label{singular-computing-PA}
Let ${{\bm X}}_l\in\mathbb{R}^{m\times d_l}$ be an orthonormal basis for the subspace $\mathcal{X}_l$ with dimension $d_l$, $l=1,2$ and suppose $d_1 \le d_2$.
If we apply singular decomposition on ${\bm X}_1^{\rm T}{\bm X}_2$ and get the thin SVD
${\bm X}_1^{\rm T}{\bm X}_2={\bm Q}_1{\bm \Sigma}{\bm Q}_2^{\rm T}$,
where ${\bm \Sigma}={\rm diag}(\sigma_1,\cdots,\sigma_{d_1})$ with $\sigma_1\ge\cdots\ge\sigma_{d_1}$.
Then the cosine of the $k$-th canonical angle $\theta_k$ between $\mathcal{X}_1$, $\mathcal{X}_2$ is defined as
$$
\cos\theta_k = \sigma_k, \quad \forall 1\le k\le d_1.
$$
Columns of ${\bm X}_1{\bm Q}_1$ and ${\bm X}_2{\bm Q}_2$ are principal vectors.
\end{lemma}


It follows from Lemma~\ref{singular-computing-PA} that
\begin{align}
\cos\theta_{1}=\max_{{{\bm x}_1\in\mathcal{X}_1}\atop{\left\|{\bm x}_1\right\|=1}}\left\|{\bm P}_{\mathcal{X}_2}({\bm x}_1)\right\|,\label{eq-PA-proj-1}\\
\cos\theta_{d_1}=\min_{{{\bm x}_1\in\mathcal{X}_1}\atop{\left\|{\bm x}_1\right\|=1}}\left\|{\bm P}_{\mathcal{X}_2}({\bm x}_1)\right\|.\label{eq-PA-proj-2}
\end{align}

Another key concept
in the statement of our main result is JL property,
which is defined as below.

\begin{definition}\citep{Foucart2017Mathematical}
A random matrix ${\bm \Phi}\in\mathbb{R}^{n\times N}$ is said to satisfy Johnson-Lindenstrauss property,
if there exists some constant $c>0$,
such that for any $\varepsilon\in(0,1/2)$ and for any ${\bm x}\in\mathbb{R}^N$,
$$
\mathbb{P}\left(\left|\left\|{\bm \Phi}{\bm x}\right\|^2-\left\|{\bm x}\right\|^2\right|>\varepsilon\left\|{\bm x}\right\|^2\right)\le 2{\rm e}^{-c\varepsilon^2n}.
$$
\end{definition}

JL property is a mild condition satisfied by many random matrices,
e.g., sub-Gaussian random matrices, partial Fourier matrices, and partial Hadamard matrices.
In addition, \citet{xingyu} asserts that JL property is implied by classical RIP for sparse signals
with sufficiently small restricted isometry constant.
Random projection with JL property is called \emph{JL random projection} in this paper.

\begin{remark}
Our analysis is based on the assumption that the dimension of a low-dimensional subspace ${\set X}_l$
remains unchanged after random projection.
For any random matrix with JL property,
this assumption is true with probability at least $1-{\rm e}^{-O(n)}$ \citep{xingyu}.
In some special cases,
such as for Gaussian random matrices,
this assumption holds almost surely.
We will use this assumption implicitly in all theorems in this paper.
\end{remark}

\section{Subspace Structure Preserving Property of JL random projection}
\label{sec:principal-angle-preserving}

In this section, we will address our main problem,
i.e., the distortion of subspace structure, or equivalently, canonical angles,
induced by JL random projection.
Based on this result, we establish a general subspace RIP
that works for any notion of subspace distance.
In addition, we compare our results with some well-known conclusions including JL Lemma
and some other similar works.

\subsection{Main Result}

Our main result, i.e.,
canonical angle preserving property of JL random projection,
 is stated in the following theorem.

\begin{theorem}\label{theorem-angle-multi}
Suppose ${\bm \Phi}\in\mathbb{R}^{n\times N}$ is a random matrix with Johnson-Lindenstrauss property, $n<N$.
Suppose ${\set X}_1, \cdots , {\set X}_L\subset\mathbb{R}^N$ are $L$ subspaces with dimensions $d_1,\cdots,d_L\le d$.
Denote by $\set{Y}_l\subset\mathbb{R}^{n}$ the image of ${\set X}_l$ under ${\bm \Phi}$, $l=1,\cdots, L$.
The $k$-th canonical angle between ${\set X}_i$, ${\set X}_j$, and ${\set Y}_i$, ${\set Y}_j$ is denoted as $\theta_k^{i,j}$ and $\psi_k^{i,j}$, respectively.
There exist positive universal constants $c_1$, $c_2$,
such that for any $1\le i,j\le L$,
any $\varepsilon\in(0, 1/2)$,
and any $n > c_1\varepsilon^{-2} \max\{d, \log L\}$, with probability at least $1- {\rm e}^{-c_2\varepsilon^2 n}$, we have
\begin{align}\label{theorem-bound-angle-multi}
\left(1 - \varepsilon\right) \theta_{k}^{i,j} \le \psi_{k}^{i,j} \le \left(1+ \varepsilon\right) \theta_{k}^{i,j},\quad \forall 1\le k\le \min\{d_i,d_j\}.
\end{align}
\end{theorem}
\begin{proof}
The proof is postponed to Section~\ref{sec:proof-theorem-multi-angle}.
\end{proof}


According to Theorem~\ref{theorem-angle-multi},
when $n=O(\varepsilon^{-2}\max\{d,\log L\})$,
each canonical angle between any two subspaces changes only by a small portion less than $\varepsilon$,
with overwhelming probability $1-{\rm e}^{-c_2\varepsilon^{2}n}$.
Thus we call Theorem~\ref{theorem-angle-multi} \emph{canonical angle preserving} (CAP) property.

As an application of the powerful Theorem~\ref{theorem-angle-multi},
we give a very short proof of a more general version of subspace RIP in \citet{xingyu}.
\begin{theorem}\label{theorem-rip-subspace}
Under the same setting as Theorem~\ref{theorem-angle-multi},
there exist positive universal constants
$c_1$, $c_2$,
such that
for any $1\le i,j\le L$,
any $\varepsilon\in(0, 1/2)$,
and
any $n > c_1\varepsilon^{-2} \max\{d, \log L\}$, with probability at least $1 - {\rm e}^{-c_2\varepsilon^2 n}$, we have
$$
(1 - \varepsilon) D({\set X}_i, {\set X}_j ) \le D({\set Y}_i, {\set Y}_j ) \le (1 + \varepsilon) D({\set X}_i, {\set X}_j ),\quad \forall d_i\le d_j,
$$
provided that subspace distance $D(\cdot,\cdot)$
can be written as a Lipschitz continuous function of the canonical angles
${\bm \theta}^{i,j}:=[\theta_1^{i,j},\cdots,\theta_{d_i}^{i,j}]^{\rm T}$
between these two subspaces,
and
\begin{align}\label{eq-condition-rip-subspace}
\liminf_{{\bm \theta}\to {\bm 0}}\frac{f({\bm \theta})}{\left\|{\bm \theta}\right\|}>0.
\end{align}
In particular, if $f$ is continuously differentiable, $f^{\prime}({\bm 0})\ge{\bm 0}$
and $f'(\bm 0)\ne\bm 0$, i.e., any entry of $f^{\prime}({\bm 0})$ is non-negative and at least one entry is positive,
then $f$ satisfies the above conditions.
\end{theorem}
\begin{proof}
Without loss of generality,
we consider two subspaces $\mathcal{X}_1$, $\mathcal X_2$.
Denote the $k$-th canonical angle between $\mathcal{X}_1$, $\mathcal{X}_2$ and $\mathcal{Y}_1$, $\mathcal{Y}_2$,
respectively, as $\theta_k$ and $\psi_k$.
According to Theorem~\ref{theorem-angle-multi}, we have
$$
\left\|{\bm \psi}-{\bm \theta}\right\|\le \varepsilon\left\|{\bm \theta}\right\|.
$$
Noticing that $f$ is Lipschitz continuous, we have
\begin{align*}
\frac{\left|f({\bm \psi})-f({\bm \theta})\right|}{f({\bm \theta})}&\le
\|f\|_{\rm Lip}\frac{\left\|{\bm \psi}-{\bm \theta}\right\|}{f({\bm \theta})}\nonumber\\
&\le \|f\|_{\rm Lip}\frac{\left\|{\bm \theta}\right\|}{f({\bm \theta})}\varepsilon,
\end{align*}
where $\|f\|_{\rm Lip}$ denotes the Lipschitz constant of $f$.
It suffices to show that $\frac{\|\bm\theta\|}{f(\bm\theta)}$ is bounded,
which follows easily from $\liminf_{\bm\theta\to\bm0}\frac{f(\bm\theta)}{\|\bm\theta\|}>0$
and the continuity of $f$.
\end{proof}

\begin{remark}
We have discussed the invariant property of some concepts about subspace structure,
namely, canonical angles and subspace distances.
In the study of subspace clustering,
another concept about subspace structure, the so-called affinity,
was proposed in \citet{soltanolkotabi2012geometric}.
The best known result on the invariance of affinity is recently presented in \citet{xingyu},
which is also an easy consequence of CAP property.
\end{remark}

\subsection{Related Works}
\label{subsec:comparison}

The statement of Theorem~\ref{theorem-angle-multi} resembles that
of JL Lemma,
which is a fundamental and valuable tool in the study of dimensionality reduction.
It states that for any set of finite data points in a high-dimensional Euclidean space,
they can be mapped to a low-dimensional space with all pairwise distances almost preserved.
The precise form of JL Lemma reads as follows.
\begin{lemma}\citep{Johnson1984Extensions}
For any set $\mathcal{V}$ of $L$ points in $\mathbb{R}^N$,
there exists a map $f:\mathbb{R}^N\rightarrow \mathbb{R}^n$, $n<N$, such that for all ${\bm x}_1$, ${\bm x}_2\in\mathcal{V}$,
$$
(1 - \varepsilon)\left\|{\bm x}_1 - {\bm x}_2\right\|^2 \le \left\|f({\bm x}_1) - f({\bm x}_2)\right\|^2 \le (1 + \varepsilon)\left\|{\bm x}_1 - {\bm x}_2\right\|^2,
$$
if $n$ is a positive integer satisfying $n \ge 4\log L/(\varepsilon^2/2 - \varepsilon^3/3)$, where $\varepsilon \in(0,1)$ is a constant.
\end{lemma}

We observe that the reduced dimension $n$ required by Theorem~\ref{theorem-angle-multi}
coincides with the requirement in JL Lemma.
For the special case $d=1$ in Theorem~\ref{theorem-angle-multi},
subspace reduces to a pair of data points lying on unit sphere,
and the required $n=O(\varepsilon^{-2}\log L)$ coincides with that in JL Lemma.

Another well-known notion related to JL Lemma is RIP for sparse signals,
which characterizes the ability of random projection to preserve
pairwise Euclidean distance between sparse signals.
Though similar in form,
our conclusion differs in many aspects from JL Lemma and RIP,
and is not a trivial extension of them.
First, Theorem~\ref{theorem-angle-multi} investigates subspaces in Euclidean space instead of points,
which makes it a valuable tool in the analysis of UoS model.
In addition, Theorem~\ref{theorem-angle-multi} focuses on canonical angles,
which better characterize relative subspace positions than any notion of subspace distance.
Furthermore, our proof deviates from that of JL Lemma and RIP for sparse signals,
and no existing RIP for point sets are invoked in the proof.

As an extension of the RIP for sparse signals,
subspace RIP has been proposed by \citet{LiGenTSP}.
The most recent and general result in this vein is presented in \citet{xingyu},
 which proves that JL random projection approximately
 preserves the projection Frobenius-norm distance between subspaces.
 This result can not be easily extended to other subspace distance definitions.
 The reason is that \citet{xingyu} studies this problem by dealing with subspace affinity as a whole,
 which is a function of projection Frobenius-norm distance,
 but has no such relationship with other subspace distance definitions.

To study subspace RIP in a systematic way,
our previous works \citet{GlobalSIP2017, Jiao2018SubspaceDSW}
study the canonical angles preserving property of Gaussian random projection.
However, the requirement on the reduced dimension $n$ in the result of \citet{GlobalSIP2017} is polynomial in the failing probability $\delta$,
which is not as rigourous as the exponential relationship in this work.
In addition, all these results
are restricted to Gaussian case,
while the result in this paper
works for a wider class of random matrices,
including partial Fourier matrices which are more useful in practice.


There are some other works that are similar to our work in form.
\citet{Eftekhari2017What} relates to this work in
studying the distortion
of the largest canonical angle
between two
tangent subspaces on the manifold, induced by a linear near-isometry map.
It discovers the relationship between such distortion
 and some geometric attributes of the manifold
 (Proposition 5).
The distortion in  this work only depends only
on the original canonical angle and failing probability.
Besides, we study each canonical angle rather than only the largest one.
\citet{Frankl1990} and \citet{ABSIL2006} study the distribution of canonical angles between random subspaces.
In both of these works,
randomness exists in the subspace itself.
While in this work, it is in the process of projection,
and thus characterize the ability of this dimensionality reduction method
to preserve subspace structure.
Finally, we remark that
\citet{Dirksen2007Dimensionality}
is easily mistaken for subspace RIP.
In fact, the target of analysis in \citet{Dirksen2007Dimensionality} is the data points
lying on union of subspaces, but not subspace itself.

\section{Compressed Subspace Learning: A Framework}
\label{sec:framework}

\subsection{Description of CSL Framework}

With the CAP property of JL random projection established,
we are now in the position to formulate the aforementioned
Compressed Subspace Learning framework,
which is done in Algorithm \ref{alg:GRPLCA}.
Note that the use of partial Fourier matrices there is not essential.
We can replaced it with any random matrices with JL property according to the application scenarios.
For example,
when the original data with dimension $N$ is sparse
or the reduced dimension is much smaller than $\log N$,
we can use unstructured random matrices,
including Gaussian and Bernoulli matrices,
which are more computational efficient than partial Fourier matrices in this case.
As long as the random matrices satisfy JL property,
Theorem~\ref{theorem-angle-multi} indicates that in Step I
the dimension of data can be compressed to $O(d)$
without destructing the UoS structure.
Thus, applying UoS-based algorithm $\mathcal A$
on compressed data will yield as good performance as
applying it on original data without compression.
This helps to circumvent the curse of dimensionality.
Note that in some case, the compression step in CSL framework
is not explicitly done after acquiring the data,
but rather be done by undersampling at the time of data-acquisition
where we acquire much less features of data than available.
Such cases are encountered, for instance, in compressive radar imaging \citep{baraniuk2007compressive}.
The CSL framework also works well for such cases.
Some concrete applications of this framework
are presented in Section~\ref{sec:app-compression}.

\begin{algorithm}[tb]
\renewcommand{\arraystretch}{1.1}
\caption{The Framework for Compressed Subspace Learning}\label{algorithm}
\begin{tabular}{l}
{\label{alg:GRPLCA}\bf Input:} \hspace{0.55em}
Original data set $\{{\bm x}_i\}_{i=1,2,\cdots}, {\bm x}_i \in \mathbb{R}^{N}$; \\
\hspace{4.5em}The dimension after compression $n, n<N$; \\
\hspace{4.5em}A selected UoS-based learning algorithm $\mathcal{A}$.\\
{\bf Output:} 
Information $\{{\bm o}_i\}_{i=1,2,\cdots}$ extracted from the input data set. \\
\hline
{\bf Step I. Applying random projection with partial Fourier matrices}\\
\hspace{3.5em} 1. Multiplying each entry of ${\bm x}_i$ by a Rademacher random variable and getting\\
\hspace{3.5em} ~~~ the sign-randomized version $\tilde{{\bm x}}_i$ of data ${\bm x}_i$, $i=1,2,\cdots$,\\
\hspace{3.5em} 2. Computing the fast Fourier transformation $\hat{{\bm x}}_i$ of the sign-randomized data $\tilde{{\bm x}}_i$,\\
\hspace{3.5em} ~~~~$i=1,2,\cdots$.\\
\hspace{3.5em} 3. Randomly sampling $n$ rows from $\hat{{\bm x}}_i$ and constructing the compressed data\\
 \hspace{3.5em} ~~~~${\bm y}_i\in\mathbb{R}^n$, $i=1,2,\cdots$.\\
{\bf Step II. Conducting the selected algorithm on the compressed data}\\
\hspace{3.5em} $\{{\bm o}_i\}_{i=1,2,\cdots} = \mathcal{A}\left(\{{\bm y}_i\}_{i=1,2,\cdots}\right)$.\\
\end{tabular}
\end{algorithm}

\subsection{Related Works}

In many problems, e.g., $\ell_2$-regression
and support vector machine (SVM) problem,
the performance of certain type of random projection as a dimensionality reduction method has been studied.
For $\ell_2$ regression problem ${\bm x}^*=\arg\min_{{\bm x}}\left\|{\bm b}-{\bm A}{\bm x}\right\|$,
it is proved that uniform sampling approximately preserves the least square solution ${\bm x}^*$ \citep{Drineas2006Sampling}.
The requirement on the reduced dimension $n$ in terms of approximation error $\varepsilon$ and
failing probability $\delta$ is $n=O(M^2\varepsilon^{-2}\log(1/\delta))$.
In the study of SVM,
\citet{Shi2012Is} discovers the almost invariant property of
margin
after Gaussian random projection,
and
gives the condition on the reduced dimension $n$ in terms of the margin distortion $\varepsilon$
and failing probability $\delta$ as $n=O(\varepsilon^{-2}\log(1/\delta))$.
\citet{Paul2013Random} considers more types of random projection,
including some of those with structured random matrices.
Different from previous works,
our study is not constrained to specific algorithms.
For example,
the framework presented in Algorithm \ref{alg:GRPLCA} is able to subsume
three very different algorithms handling different problems
presented in Section~\ref{sec:app-compression}.
Such universality is made possible only by
the powerful mathematical engine of CAP property.
With this powerful engine,
it is possible to adopt the CSL framework to
handle many other subspace-related problems
and give a performance analysis.

\section{The Proof of Theorem~\ref{theorem-angle-multi}}
\label{sec:proof-theorem-multi-angle}

\subsection{Reducing to the case $L=2$}
\label{subsubsec:proof-fromLto2}

This part is standard.
Assume the conclusion of Theorem~\ref{theorem-angle-multi} is true for $L=2$.
For general $L$, it follows from this special case that
 \eqref{theorem-bound-angle-multi} holds with probability at least
$1-\frac{L(L-1)}{2}{\rm e}^{-c_{2}\varepsilon^2n}$,
given that $n>c_1\varepsilon^{-2}d$.
Thus for $n>\frac{2}{c_2\varepsilon^2}\log(L(L-1))$ and $n>c_1\varepsilon^{-2}d$,
inequality \eqref{theorem-bound-angle-multi} holds with probability at least
$1-\mathrm e^{-c_2\varepsilon^2n/2}$.
The conclusion follows by adjusting the values of $c_1$ and $c_2$.

\subsection{A Two-sided Bound of $\sin\psi_k$}
\label{subsec:proof-lemma-step1}

We begin by giving $\sin\psi_k$ a two-sided bound which is easier to handle.
\begin{lemma}\label{lemma-step1}
Under the same setting as Theorem~\ref{theorem-angle-multi},
assume the $k$-th principal vectors between ${\set X}_1$ and ${\set X}_2$ are given by ${\bm u}_{1,k}\in\mathcal{X}_1$ and ${\bm u}_{2,k}\in\mathcal{X}_2$.
Denote $\varphi_{1,k,k}$ as the largest canonical angle between
subspace $\mathcal{Y}_{1,1:k}:={\rm span}\{{\bm \Phi}{\bm u}_{1,1},\cdots,{\bm \Phi}{\bm u}_{1,k}\}$
and $\mathcal{Y}_2$.
Denote $\varphi_{k,d_1,1}$ as the smallest canonical angle between
subspace $\mathcal{Y}_{1,k:d_1}:={\rm span}\{{\bm \Phi}{\bm u}_{1,k},\cdots,{\bm \Phi}{\bm u}_{1,d_1}\}$
and $\mathcal{Y}_2$.
The sine of the $k$-th projected canonical angle $\psi_k$ between the projected subspaces ${\set Y}_1$ and ${\set Y}_2$ is bounded by
$$
\sin\varphi_{k,d_1,1}\le\sin\psi_k\le\sin\varphi_{1,k,k}.
$$
\end{lemma}

According to the definition of canonical angles,
dealing with the largest canonical angle $\varphi_{1,k,k}$
or the smallest canonical angle $\varphi_{k,d_1,1}$
is much easier than dealing with the $k$-th canonical angle $\psi_k$.
The reason is that $\psi_k$ is recursively defined and it relies on $k$ pairs of
principal vectors.
While the calculation of $\varphi_{1,k,k}$ and $\varphi_{k,d_1,1}$ only involves solving a maximization or
minimization problem shown in \eqref{eq-PA-proj-1} and \eqref{eq-PA-proj-2}.
Thus the bound provided in Lemma \ref{lemma-step1} is much easier to handle than $\sin\psi_k$.

\begin{proof}
The proof of Lemma~\ref{lemma-step1} is an application of von Neumann min-max theorem.


We first establish the relationship between $\sin\psi_k$ and $\sin\varphi_{1,k,k}$.
Denote ${\bm V}_i$ as the orthonormal basis of $\set{Y}_i$, $i=1,2$.
We calculate $\sin\psi_k$ via von Neumann min-max theorem as below.
\begin{align}
\sin^2\psi_k&=1-\sigma_k^2\left({\bm V}_2^{\rm T}{\bm V}_1\right)\nonumber\\
&=1-\max_{\bm{S}^{k-1} \subset \mathbb{R}^{d_1}}\min_{{\bm x}\in \bm{S}^{k-1}}\left\|{\bm V}_2^{\rm T}{\bm V}_1{\bm x}\right\|^2,\quad \forall 1\le k\le d_1.\label{forcounterpart1}
\end{align}
Denote the orthonormal basis of the $k$-dimensional subspace spanned by $\bm{S}^{k-1}$ as ${\bm Q}\in\mathbb{R}^{d_1\times k}$.
We have $\{{\bm V}_1{\bm x}:{\bm x}\in\bm{S}^{k-1}\}=\{{\bm y}:{\bm y}\in\mathcal{C}({\bm V}_1{\bm Q}),\left\|{\bm y}\right\|=1\}$,
where $\mathcal{C}({\bm V}_1{\bm Q})$ denotes the column space of matrix ${\bm V}_1{\bm Q}$.
Replacing ${\bm V}_1{\bm x}$ with ${\bm y}$ in \eqref{forcounterpart1}, we have
\begin{align*}
\sin^2\psi_k
&=1-\max_{{\bm Q} \in \mathbb{R}^{d_1\times k}}\min_{{{\left\|{\bm y}\right\|=1}\atop{{\bm y}\in \mathcal{C}({\bm V}_1{\bm Q})}}}\left\|{\bm V}_2^{\rm T}{\bm y}\right\|^2,\quad \forall 1\le k\le d_1.
\end{align*}
Noticing that $\left\|{\bm V}_2^{\rm T}{\bm y}\right\|$ is the norm of the projection of ${\bm y}$ onto ${\set Y}_2$,
and the norm of ${\bm y}$ equals $1$, we can further
simplify the above expression as
\begin{align}
\sin^2\psi_k
&=1-\max_{{\bm Q} \in \mathbb{R}^{d_1\times k}}\min_{{{\left\|{\bm y}\right\|=1}\atop{{\bm y}\in \mathcal{C}({\bm V}_1{\bm Q})}}}\left\|{\bm P}_{{\set Y}_2}({\bm y})\right\|^2\nonumber\\
&=\min_{{\bm Q} \in \mathbb{R}^{d_1\times k}}\max_{{{\left\|{\bm y}\right\|=1}\atop{{\bm y}\in \mathcal{C}({\bm V}_1{\bm Q})}}}\left\|{\bm P}_{{\set Y}_2^{\perp}}({\bm y})\right\|^2.\label{eq3}
\end{align}

Now the RHS  of \eqref{eq3} is the projection norm of some unit vector.
By taking ${\bm Q}$ as
${\bm Q}_u:=\left[{\bm V}_1^{\rm T}{\bm a}_{1,1},\cdots,{\bm V}_1^{\rm T}{\bm a}_{1,k}\right]$,
and plugging \eqref{eq-PA-proj-1}, we have
\begin{align}\label{forcounterpart4-temp}
\sin^2\psi_k&\le \max_{{{\left\|{\bm y}\right\|=1}\atop{{\bm y}\in \mathcal{C}({\bm V}_1{\bm Q}_u)}}}\left\|{\bm P}_{{\set Y}_2^{\perp}}({\bm y})\right\|^2\nonumber\\
&=\max_{{{\left\|{\bm y}\right\|=1}\atop{{\bm y}\in {\set{Y}_{1,1:k}}}}}\left\|{\bm P}_{{\set Y}_2^{\perp}}({\bm y})\right\|^2\nonumber\\
&=\sin^2\varphi_{1,k,k}.
\end{align}

%

It is the turn of $\sin\varphi_{k,d_1,1}$.
We need to discover its connection with $\sin\psi_k$.
To this end,
we
derive step by step the counterparts of \eqref{forcounterpart1}, \eqref{eq3}, and \eqref{forcounterpart4-temp} as
\begin{align}
\sin^2\psi_k&=1-\sigma_k^2\left({\bm V}_2^{\rm T}{\bm V}_1\right)\nonumber\\
&=1-\min_{\bm{S}^{d_1-k} \subset \mathbb{R}^{d_1}}\max_{{\bm x}\in \bm{S}^{d_1-k}}\left\|{\bm V}_2^{\rm T}{\bm V}_1{\bm x}\right\|^2\nonumber\\
&=\max_{{\bm Q} \in \mathbb{R}^{d_1\times (d_1-k+1)}}\min_{{{\left\|{\bm y}\right\|=1}\atop{{\bm y}\in \mathcal{C}({\bm V}_1{\bm Q})}}}\left\|{\bm P}_{{\set Y}_2^{\perp}}({\bm y})\right\|^2\nonumber\\
&\ge\min_{{{\left\|{\bm y}\right\|=1}\atop{\bm y}\in \mathcal{C}({\bm V}_1{\bm Q}_l)}}\left\|{\bm P}_{{\set Y}_2^{\perp}}({\bm y})\right\|^2\nonumber\\
&=\min_{{\left\|{\bm y}\right\|=1}\atop{{\bm y}\in {\set{Y}_{1,k:d_1}}}}\left\|{\bm P}_{{\set Y}_2^{\perp}}({\bm y})\right\|^2\nonumber\\
&=\sin^2\varphi_{k,d_1,1},\quad \forall 1\le k\le d_1,\label{forcounterpart5-temp}
\end{align}
where
$
{\bm Q}_l:=[{\bm V}_1^{\rm T}{\bm a}_{1,k},\cdots,{\bm V}_1^{\rm T}{\bm a}_{1,d_1}],
$
and
equation \eqref{forcounterpart5-temp} uses the relationship between the projection norm and the smallest canonical angle shown in \eqref{eq-PA-proj-2}.
\end{proof}

\subsection{Proof of the Canonical Angle Sine Preserving Property}
\label{subsec:proof-lemma-sin}

In this subsection,
we will prove the following lemma based on Lemma~\ref{lemma-step1}.
\begin{lemma}\label{lemma-sin}
Under the same setting as Theorem~\ref{theorem-angle-multi},
there exist positive universal constants $c_1$,
$c_2$,
such that
for any $1\le i,j\le L$,
any $\varepsilon\in(0,1/2)$,
and any $n>c_1\varepsilon^{-2}\max\{d,\log L\}$,
with probability at least $1-{\rm e}^{-c_2\varepsilon^2 n}$, we have
\begin{equation}\label{theorem-bound-sin}
    \left(1-\varepsilon\right)\sin\theta_k^{i,j} \le \sin\psi_k^{i,j} \le \left(1+\varepsilon\right)\sin\theta_k^{i,j}, \quad \forall 1\le k\le \min\{d_i,d_j\}.
\end{equation}
\end{lemma}

Thanks to the proof in Section~\ref{subsubsec:proof-fromLto2},
it suffices to consider canonical angles
between two subspaces $\mathcal{X}_1$, $\mathcal{X}_2$.
For convenience, we omit the superscript
and denote the $k$-th original (resp. projected) canonical angle as $\theta_k$ (resp. $\psi_k$).

There are $d_1$ canonical angles with the assumption $d_1\le d_2$.
We will prove that the $k$-th canonical angle satisfies \eqref{theorem-bound-sin}.
With a similar argument in Section~\ref{subsubsec:proof-fromLto2},
this would suffice to show that \eqref{theorem-bound-sin}
holds simultaneously for all $k$.

Therefore, we only need to consider the $k$-th canonical angle,
where $k$ is an integer in $[1,d_1]$.
According to Lemma~\ref{lemma-step1},
we only need to prove
\begin{align}
\mathbb{P}\left(\sin\varphi_{1,k,k}\le (1+\varepsilon)\sin\theta_k\right)&\ge 1-{\rm e}^{-c_2\varepsilon^2n},\label{eq-step21}\\
\mathbb{P}\left(\sin\varphi_{k,d_1,1}\ge (1-\varepsilon)\sin\theta_k\right)&\ge 1-{\rm e}^{-c_2\varepsilon^2n}.\label{eq-step22}
\end{align}

We will complete these in the following two parts.
We will follow the notations ${\bm u}_{i,k}$, $\mathcal Y_{1,1:k}$,
$\mathcal Y_{1,k:d_1}$, $\varphi_{1,k,k}$ and $\varphi_{k,d_1,1}$
in Lemma~\ref{lemma-step1}.
We further denote ${{\bm U}}_i:=[{\bm u}_{i,1},\cdots,{\bm u}_{i,d_i}]$
as an orthonormal basis for the original subspace $\mathcal{X}_i$,
and ${\bm A}_i:={\bm \Phi}{\bm U}_i$
as a basis for the projected subspace ${\set Y}_i$,
where the $k$-th column ${\bm a}_{i,k}$ is obtained by
${\bm a}_{i,k}:={\bm \Phi}{\bm u}_{i,k}$, $k=1,\cdots,d_i$, $i=1,2$.


\subsubsection{Proof of \eqref{eq-step21}}
\label{subsubsec:proof-eq-step21}

According to \eqref{eq-PA-proj-1}
both $\sin\varphi_{1,k,k}$ and $\sin\theta_k$ can be written as
the solution of a maximum problem as below.
\begin{align*}
\sin^2\varphi_{1,k,k}&= \max_{{\bm x}\in\bm{S}^{k-1}}\left\|{\bm P}_{{\set Y}_2^{\perp}}\left(\frac{{\bm A}_{1,1:k}{\bm x}}{\left\|{\bm A}_{1,1:k}{\bm x}\right\|}\right)\right\|^2,\\
\sin^2\theta_{k}&= \max_{{\bm x}\in\bm{S}^{k-1}}\left\|{\bm P}_{{\set X}_2^{\perp}}({\bm U}_{1,1:k}{\bm x})\right\|^2,
\end{align*}
where ${\bm A}_{1,1:k}:=[{\bm a}_{1,1},\cdots,{\bm a}_{1,k}]$
and ${\bm U}_{1,1:k}:=[{\bm u}_{1,1},\cdots,{\bm u}_{1,k}]$.
Then
\begin{align}\label{forcounterpart4}
\sin^2\varphi_{1,k,k}-\sin^2\theta_{k}&= \max_{{\bm x}\in\bm{S}^{k-1}}\left\|{\bm P}_{{\set Y}_2^{\perp}}\left(\frac{{\bm A}_{1,1:k}{\bm x}}{\left\|{\bm A}_{1,1:k}{\bm x}\right\|}\right)\right\|^2-\sin^2\theta_k\nonumber\\
&\le \max_{{\bm x}\in\bm{S}^{k-1}}\left[ \left\|{\bm P}_{{\set Y}_2^{\perp}}\left(\frac{{\bm A}_{1,1:k}{\bm x}}{\left\|{\bm A}_{1,1:k}{\bm x}\right\|}\right)\right\|^2
- \left\|{\bm P}_{{\set X}_2^{\perp}}\left({\bm U}_{1,1:k}{\bm x}\right)\right\|^2 \right].
\end{align}
The RHS of \eqref{forcounterpart4} involves the maximum over the whole sphere $\bm{S}^{k-1}$,
which can be handled by a standard entropy argument (Appendix \ref{app:covering-argument}).
We may take an $\varepsilon$-net of the unit sphere.
Then it suffices to consider any given ${\bm x}$,
and use union bound to complete proof.

Here we need to invoke the following lemma about the perturbation on orthonormal basis.
\begin{lemma}(\citet{xingyu}, Lemma 6)\label{lemma:orthonormal-preserving}
Suppose ${\bm U}$ is an $N\times d$ matrix with orthonormal columns,
i.e., ${\bm U}^{\rm T}{\bm U}={\bm I}_d$.
Let ${\bm \Phi}$ be an $n\times N$ random matrix with JL property.
Then there exist positive universal constants $c_1$, $c_2$,
such that for any $\varepsilon<(0, 1/2)$,
any $n>c_1\varepsilon^{-2}d$, we have
\begin{equation*}
P(1-\varepsilon<\sigma_{d}({\bm \Phi}{\bm U})\le \sigma_1({\bm \Phi}{\bm U})<1+\varepsilon)\ge 1-{\rm e}^{-c_2\varepsilon^2 n}.
\end{equation*}
\end{lemma}

 Plugging $\sigma_k({\bm A}_{1,1:k})\ge (1-\varepsilon)$ into \eqref{forcounterpart4}, we have
 \begin{align}\label{eq-temp-20}
 \sin^2\varphi_{1,k,k}-\sin^2\theta_k\le \max_{{\bm x}\in\bm{S}^{k-1}}\left[ (1-\varepsilon)^{-2}\left\|{\bm P}_{{\set Y}_2^{\perp}}\left({\bm A}_{1,1:k}{\bm x}\right)\right\|^2
- \left\|{\bm P}_{{\set X}_2^{\perp}}\left({\bm U}_{1,1:k}{\bm x}\right)\right\|^2  \right].
 \end{align}
Notice that the RHS of \eqref{eq-temp-20} is equal to
$$
{\bm x}^{\rm T}\left((1-\varepsilon)^{-2}{\bm A}_{1,1:k}^{\rm T}{\bm P}_{{\set Y}_2^{\perp}}\left({\bm A}_{1,1:k}\right)
-{\bm U}_{1,1:k}^{\rm T}{\bm P}_{{\set X}_2^{\perp}}({\bm U}_{1,1:k})\right){\bm x}.
$$
Following a standard covering argument (Appendix \ref{app:covering-argument}),
we can evaluate the RHS of \eqref{eq-temp-20} by calculating on
a $\frac14$-net $\mathcal{N}$.
 \begin{align}\label{eq-calcualte-on-net}
 \sin^2\varphi_{1,k,k}-\sin^2\theta_k\le 2\max_{{\bm x}\in\mathcal{N}}\left[  \left|(1-\varepsilon)^{-2}\left\|{\bm P}_{{\set Y}_2^{\perp}}\left({\bm A}_{1,1:k}{\bm x}\right)\right\|^2
- \left\|{\bm P}_{{\set X}_2^{\perp}}\left({\bm U}_{1,1:k}{\bm x}\right)\right\|^2\right|  \right].
 \end{align}

Now it suffices to bound the maximum of the last quantity over the $\frac14$-net $\mathcal{N}$.
We only need to consider the upper bound of the bracket expression for any given ${\bm x}\in\mathcal{N}$.
Noticing that
${\bm A}_{1,1:k}{\bm x}={\bm \Phi}{\bm U}_{1,1:k}{\bm x}$
is compressed from ${\bm U}_{1,1:k}{\bm x}$,
the bracket expression is similar to the distortion on ${\bm P}_{\mathcal{X}_2^{\perp}}\left({\bm U}_{1,1:k}{\bm x}\right)$ induced by JL random projection,
which is given by the following Lemma.
\begin{lemma}(\citet{xingyu}, Lemma 2)\label{lemma-line}
Suppose ${\bm \Phi}$ is a random matrix with JL property.
Suppose ${\bm x}_1$, $\set X_2$ are respectively a vector and a $d$-dimensional
subspace of $\mathbb R^N$.
Denote ${\bm y}_1:={\bm \Phi}{\bm x}_1/\left\|{\bm \Phi}{\bm x}_1\right\|$ and $\set Y_2$ as
the projection of $\set X_2$ with ${\bm \Phi}$.
Then there exist positive universal constants $c_1$, $c_2$,
such that
for any $\varepsilon\in(0,1/2)$ and any $n>c_1\varepsilon^{-2}d$, we have
\begin{equation*}
    \left|\left\|{\bm P}_{\mathcal{Y}_2^{\perp}}({\bm y}_1)\right\|^2-\left\|{\bm P}_{\mathcal{X}_2^{\perp}}({\bm x}_1)\right\|^2\right|\le\varepsilon\left\|{\bm P}_{\mathcal{X}_2^{\perp}}({\bm x}_1)\right\|^2
\end{equation*}
with probability at least $1-{\rm e}^{-c_2\varepsilon^2n}$.
\end{lemma}

To use Lemma~\ref{lemma-line}, we need to reformulate \eqref{eq-calcualte-on-net} as below.
 \begin{align}
 &\left|(1-\varepsilon)^{-2}\left\|{\bm P}_{{\set Y}_2^{\perp}}\left({\bm A}_{1,1:k}{\bm x}\right)\right\|^2
- \left\|{\bm P}_{{\set X}_2^{\perp}}\left({\bm U}_{1,1:k}{\bm x}\right)\right\|^2\right|\nonumber\\
=&\left\|{\bm A}_{1,1:k}{\bm x}\right\|^2\left|(1-\varepsilon)^{-2}\left\|{\bm P}_{{\set Y}_2^{\perp}}\left(\frac{{\bm A}_{1,1:k}{\bm x}}{\left\|{\bm A}_{1,1:k}{\bm x}\right\|}\right)\right\|^2
-\left\|{\bm A}_{1,1:k}{\bm x}\right\|^{-2} \left\|{\bm P}_{{\set X}_2^{\perp}}\left({\bm U}_{1,1:k}{\bm x}\right)\right\|^2\right|\nonumber\\
\le&(1+\varepsilon)^2\left|(1-\varepsilon)^{-2}\left\|{\bm P}_{{\set Y}_2^{\perp}}\left(\frac{{\bm A}_{1,1:k}{\bm x}}{\left\|{\bm A}_{1,1:k}{\bm x}\right\|}\right)\right\|^2
-\left\|{\bm A}_{1,1:k}{\bm x}\right\|^{-2} \left\|{\bm P}_{{\set X}_2^{\perp}}\left({\bm U}_{1,1:k}{\bm x}\right)\right\|^2\right|\label{eq-temp-12}.
\end{align}
 Now we can invoke Lemma~\ref{lemma-line} and get
\begin{align}
&\left|(1-\varepsilon)^{-2}\left\|{\bm P}_{{\set Y}_2^{\perp}}\left(\frac{{\bm A}_{1,1:k}{\bm x}}{\left\|{\bm A}_{1,1:k}{\bm x}\right\|}\right)\right\|^2
-\left\|{\bm A}_{1,1:k}{\bm x}\right\|^{-2} \left\|{\bm P}_{{\set X}_2^{\perp}}\left({\bm U}_{1,1:k}{\bm x}\right)\right\|^2\right|\nonumber\\
\le&  \left(1-\varepsilon\right)^{-2}\left|\left\|{\bm P}_{{\set Y}_2^{\perp}}\left(\frac{{\bm A}_{1,1:k}{\bm x}}{\left\|{\bm A}_{1,1:k}{\bm x}\right\|}\right)\right\|^2
- \left\|{\bm P}_{{\set X}_2^{\perp}}\left({\bm U}_{1,1:k}{\bm x}\right)\right\|^2\right|\nonumber\\
&+ \left|\left(1-\varepsilon\right)^{-2}-\left\|{\bm A}_{1,1:k}{\bm x}\right\|^{-2}\right|\left\|{\bm P}_{{\set X}_2^{\perp}}\left({\bm U}_{1,1:k}{\bm x}\right)\right\|^2\label{forcounterpart2}\\
<& 4\varepsilon\left\|{\bm P}_{\mathcal{X}_2^{\perp}}({\bm U}_{1,1:k}{\bm x})\right\|^2+8\varepsilon\left\|{\bm P}_{{\set X}_2^{\perp}}\left({\bm U}_{1,1:k}{\bm x}\right)\right\|^2\nonumber\\
= & 12\varepsilon\sin^2\theta_k,\label{eq-temp-21}
 \end{align}
where $\varepsilon<1/2$ is used implicitly.

Thus, for any $\varepsilon\in(0,1/2)$, with probability at least $1-9^d{\rm e}^{-c_2\varepsilon^2 n}$,
we have
$$
\sin^2\varphi_{1,k,k}-\sin^2\theta_k\le 2\cdot \frac94 \cdot 12\varepsilon\sin^2\theta_k.
$$
Redefining $c_1:=\max\{c_1, 2\log 9/c_2\}$, $c_2:=(c_2-\log 9/c_1)/54^2$,
we have \eqref{eq-step21} hold.

\subsubsection{Proof of \eqref{eq-step22}}
\label{subsubsec:proof-eq-step22}

In this proof, we follow the same approach as Section~\ref{subsubsec:proof-eq-step21}.
According to \eqref{eq-PA-proj-2}, we have
\begin{align*}
\sin^2\varphi_{k,d_1,1}&=\min_{{\bm x}\in\bm{S}^{d_1-k}}\frac{\left\|{\bm P}_{{\set Y}_2^{\perp}}\left({\bm A}_{1,k:d_1}{\bm x}\right)\right\|^2}{\left\|{\bm A}_{1,k:d_1}{\bm x}\right\|^2},\\
\sin^2\theta_k&=\min_{{\bm x}\in\bm{S}^{d_1-k}}\left\|{\bm P}_{{\set X}_2^{\perp}}\left({\bm U}_{1,k:d_1}{\bm x}\right)\right\|^2,
\end{align*}
where ${\bm A}_{1,k:d_1}:=[{\bm a}_{1,k},\cdots,{\bm a}_{1,d_1}]$
and ${\bm U}_{1,k:d_1}:=[{\bm u}_{1,k},\cdots,{\bm u}_{1,d_1}]$.
Then we could derive the counterpart of \eqref{forcounterpart4}, \eqref{eq-temp-20}, \eqref{eq-calcualte-on-net}, and \eqref{eq-temp-12} as below.
\begin{align}
\sin^2\varphi_{k,d_1,1}-\sin^2\theta_{k}&\ge \min_{{\bm x}\in\bm{S}^{d_1-k}}\left[ \left\|{\bm P}_{{\set Y}_2^{\perp}}\left(\frac{{\bm A}_{1,k:d_1}{\bm x}}{\left\|{\bm A}_{1,k:d_1}{\bm x}\right\|}\right)\right\|^2
- \left\|{\bm P}_{{\set X}_2^{\perp}}\left({\bm U}_{1,k:d_1}{\bm x}\right)\right\|^2 \right]\nonumber\\
&=-\max_{{\bm x}\in\bm{S}^{d_1-k}}\left[  \left\|{\bm P}_{{\set X}_2^{\perp}}\left({\bm U}_{1,k:d_1}{\bm x}\right)\right\|^2
- \left\|{\bm P}_{{\set Y}_2^{\perp}}\left(\frac{{\bm A}_{1,k:d_1}{\bm x}}{\left\|{\bm A}_{1,k:d_1}{\bm x}\right\|}\right)\right\|^2\right]\nonumber\\
&\ge  -\max_{{\bm x}\in\bm{S}^{d_1-k}}\left[  \left\|{\bm P}_{{\set X}_2^{\perp}}\left({\bm U}_{1,k:d_1}{\bm x}\right)\right\|^2
-(1+\varepsilon)^{-2}\left\|{\bm P}_{{\set Y}_2^{\perp}}\left({\bm A}_{1,k:d_1}{\bm x}\right)\right\|^2  \right]\nonumber\\
&\ge  -2\max_{{\bm x}\in\mathcal{N}}\left|\left\|{\bm P}_{{\set X}_2^{\perp}}\left({\bm U}_{1,k:d_1}{\bm x}\right)\right\|^2
- (1+\varepsilon)^{-2}\left\|{\bm P}_{{\set Y}_2^{\perp}}\left({\bm A}_{1,k:d_1}{\bm x}\right)\right\|^2 \right|\nonumber\\
&\ge -2\left(1+\varepsilon\right)^2\max_{{\bm x}\in\mathcal{N}}\left|
\frac{\left\|{\bm P}_{{\set X}_2^{\perp}}\left({\bm U}_{1,k:d_1}{\bm x}\right)\right\|^2}{\left\|{\bm A}_{1,k:d_1}{\bm x}\right\|^2}
-  \left(1+\varepsilon\right)^{-2}\left\|{\bm P}_{{\set Y}_2^{\perp}}\left(\frac{{\bm A}_{1,k:d_1}{\bm x}}{\left\|{\bm A}_{1,k:d_1}{\bm x}\right\|}\right)\right\|^2\right|,\nonumber
\end{align}
where $\mathcal{N}$ here denotes the $\frac14$-net of unit sphere $\bm{S}^{d_1-k}$.

Now we consider each given ${\bm x}\in\mathcal{N}$.
We could derive the counterpart of \eqref{forcounterpart2} and \eqref{eq-temp-21} as below
 \begin{align*}
&\left|
\frac{\left\|{\bm P}_{{\set X}_2^{\perp}}\left({\bm U}_{1,k:d_1}{\bm x}\right)\right\|^2}{\left\|{\bm A}_{1,k:d_1}{\bm x}\right\|^2}
-  \left(1+\varepsilon\right)^{-2}\left\|{\bm P}_{{\set Y}_2^{\perp}}\left(\frac{{\bm A}_{1,k:d_1}{\bm x}}{\left\|{\bm A}_{1,k:d_1}{\bm x}\right\|}\right)\right\|^2\right|\\
  \le& \left((1-\varepsilon)^{-2}-(1+\varepsilon)^{-2}\right)\left\|{\bm P}_{{\set X}_2^{\perp}}\left({\bm U}_{1,k:d_1}{\bm x}\right)\right\|^2
+ \left(1+\varepsilon\right)^{-2}\left\|{\bm P}_{\mathcal{X}_2^{\perp}}({\bm U}_{1,k:d_1}{\bm x}_1)\right\|^2\varepsilon\\
\le& 9\varepsilon\sin^2\theta_k.
 \end{align*}

Following the same argument as the end of Section~\ref{subsubsec:proof-eq-step21},
we could complete the proof.

\subsection{Proof of Theorem~\ref{theorem-angle-multi}}
\label{subsec:proof-theorem-angle-multi-from-lemma-angle}

According to Section~\ref{subsubsec:proof-fromLto2},
it suffices to prove Theorem~\ref{theorem-angle-multi} with $L=2$.
For convenience, we also use $\theta_k$ (resp. $\psi_k$) instead of
$\theta_k^{1,2}$ (resp. $\psi_k^{1,2}$) to denote
the $k$-th original (resp. projected) canonical angle.
Then we only need to prove
\begin{equation}\label{lemma-bound-angle}
    \left(1-\varepsilon\right)\theta_k \le \psi_k \le \left(1+\varepsilon\right)\theta_k,\quad \forall 1\le k\le d_1.
\end{equation}

According to Lemma~\ref{lemma-sin}, for any $\delta \in (0, 1/2)$,
there exist positive universal constants $c_1$, $c_2$, such that for any $n>c_1\delta^{-2}d$, with probability at least $1-{\rm e}^{-c_2\delta^{2} n}$, we have
$$
(1-\delta)\sin\theta_k \le \sin\psi_k\le(1+\delta)\sin\theta_k.
$$

We consider two cases: $\theta_k\in[0,\pi/4]$ and $\theta_k\in[\pi/4,\pi/2]$.
When $\theta_k\in[0,\pi/4]$, we have
\begin{align*}
\left|\psi_k - \theta_k\right| &\le \left|\arcsin\left(\left(1\pm\delta\right)\sin\theta_k\right) -  \arcsin\left(\sin\theta_k\right)\right|\nonumber\\
&\le \frac{\delta\sin\theta_k}{\sqrt{1 - \xi_k^2}},
\end{align*}
where
$\left(1-\delta\right)\sin\theta_k<\xi_k<\left(1+\delta\right)\sin\theta_k$.
Assume $\delta\le 1/3$, we have
\begin{align*}
\left|\psi_k - \theta_k\right| &\le \frac{\delta\sin\theta_k}{\sqrt{1 - (1+\delta)^2\sin^2\theta_k}}\\
&\le  \frac{\delta\sin\theta_k}{\sqrt{1 - (4/3)^2\sin^2\theta_k}}\\
&\le 3\delta\theta_k.
\end{align*}

When $\theta_k\in[\pi/4,\pi/2]$,
considering that the function $\arcsin$ is uniformly continuous within interval $[1/2,1]$,
for any $\epsilon\in(0,1/2)$, there exists constant $\delta_\epsilon>0$,
such that for any $\left|\sin\psi_k-\sin\theta_k\right|<\delta_\epsilon$,
we have
$$
\left|\theta_k-\psi_k\right|\le \epsilon\theta_k.
$$
Redefining $\varepsilon:=\max\{\epsilon,3\delta\}$,
we can get \eqref{lemma-bound-angle}.

One should note that in the second case, i.e., $\theta_k\in[\pi/4,\pi/2]$,
we fail to prove the reduced dimension $n=O(\varepsilon^{-2})$.
To get this conclusion, we need to first prove
$$
\left|\cos^2\psi_k-\cos^2\theta_k\right|\le C_1\varepsilon\cos\theta_k+C_2\varepsilon^2,
$$
and then use the $\arccos$ function to get this result.
Since the proof is similar with that of Lemma~\ref{lemma-sin},
we postpone it to Appendix \ref{app:proof-cosine-preserving}.

\section{Compressed Subspace Learning: Applications}
\label{sec:app-compression}

In this section,
we will show instances of the CSL framework proposed in Section~\ref{sec:framework} on three subspace-related tasks, namely,
subspace visualization, active subspace detection, and subspace clustering.
Through the instances on the first two tasks, we validate the performance preserving property of JL random projection.
The related theoretical analyses are deferred to appendix.
For subspace clustering,
some of its algorithms suffer from the high computational complexity.
We empirically verify
that JL random projection
can not only approximately preserve the clustering accuracy,
but also significantly reduce the time consumption.
For convenience, we call the algorithm with JL random projection as the compressed version,
such as compressed subspace clustering.

\subsection{Data sets}

We will use the following two real-wold data sets to test the performance of CSL framework.

{\bf YaleB Face data set} consists of frontal face images of $38$ human subjects under $64$ different illumination conditions.
The size of images is $192\times 168$.
We reshape each image to a vector of $32256$ dimensions.
It is assumed that images of the same subject lie in a $9$-dimensional subspace \citep{Noisy2016Wang}.
For convenience, in the following experiments, we randomly select $4$ subjects to analyze their face images.

{\bf Webb Spam Corpus 2006 data set} is a collection of approximately $350,000$ spam web pages,
which are divided into two categories.
The feature, extracted following the way provided by LIBSVM%
\footnote{https://www.csie.ntu.edu.tw/~cjlin/libsvmtools/datasets/binary.html},
is a $16$ million-dimensional sparse vector.
In the following experiments, we uniformly sample $1000$ data points within each category,
and model them with UoS.
Specifically, we assume that data points belonging to the same category lie in a $10$-dimensional subspace%
\footnote{The dimension of subspace is assumed to be ten because the energy for these two categories on the first $10$ principal components account for $79.1\%$ and $67.7\%$,
which is large enough.}.
We remark that applying some subspace learning algorithms, e.g., subspace clustering,
on this data set directly is infeasible for extremely high ambient dimension.
Under such circumstance, dimensionality reduction is not only beneficial, but also necessary.

Considering that data in Webb Spam Corpus 2006 data set is sparse,
we use Bernoulli random matrices instead of partial Fourier matrices in Step I in Algorithm \ref{alg:GRPLCA}
to improve the computational efficiency.
While for YaleB Face data set, the data within it is dense,
thus partial Fourier matrices is used.

\subsection{Compressed Subspace Visualization}
\label{subsec:CSV}

\begin{figure}[t]
\subfigure[]{
\begin{minipage}{\linewidth}
  \centering
  \includegraphics[width=0.9\textwidth]{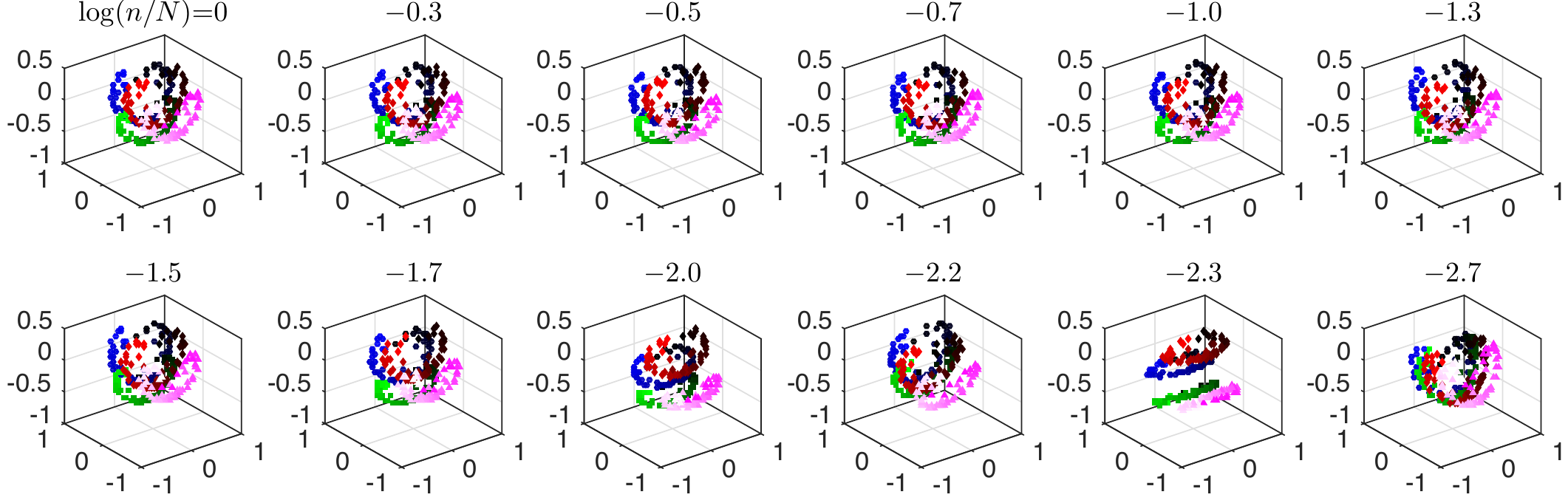}
  \centering
\end{minipage}
}
\subfigure[]{
\begin{minipage}{\linewidth}
  \centering
  \includegraphics[width=0.9\textwidth]{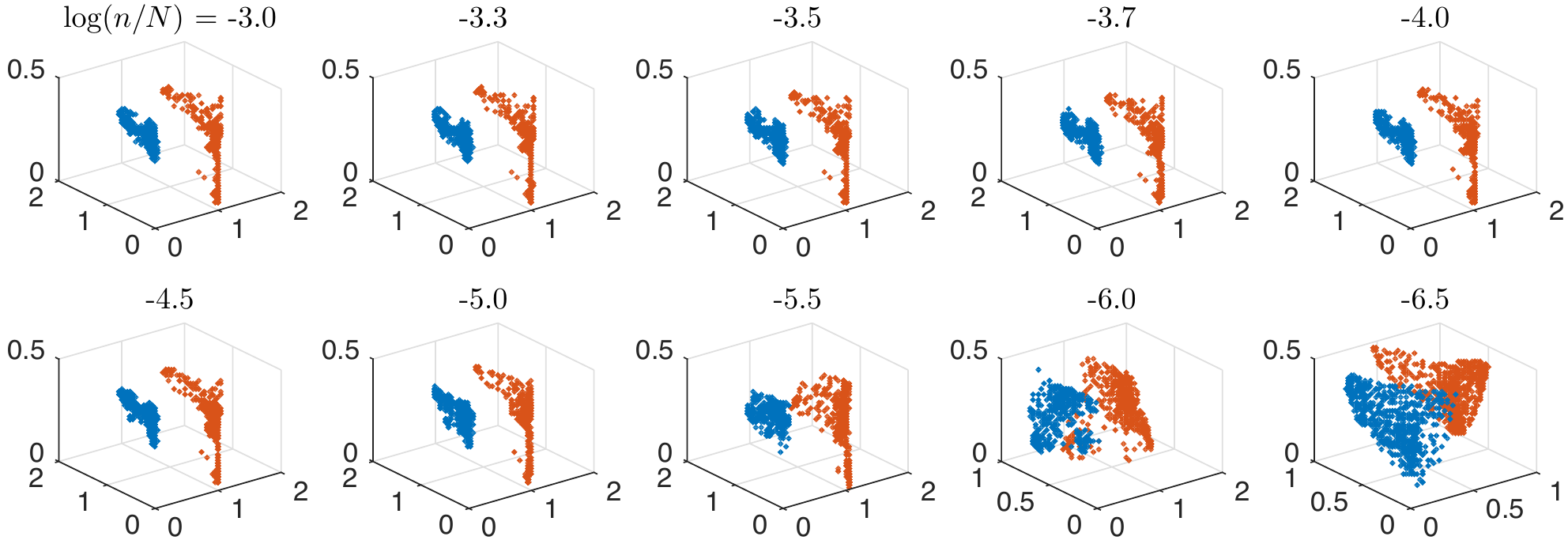}
  \centering
\end{minipage}
}
\caption{Compressed angle-based subspace visualization results under different reduced dimensions $n$ on (a) YaleB Face data set and (b) Webb Spam Corpus 2006 data set. 
Different colors represent various categories.
}
\label{fig:webb-CSV}
\end{figure}

Data visualization is an effective way to help people understand high-dimensional data.
Different from 3-dimensional space,
data in high-dimensional space can not be depicted directly,
which brings difficulty in the direct understanding.
Data visualization tries to represent data in a visual context to
help people quickly capture the main relationship between data points,
and acquire enough information.

Subspace visualization is designed for UoS model.
One method proposed in \citet{Shen2018Subspace} is based on canonical angles,
and thus called angle-based subspace visualization.
As its name suggests, the visualization result is determined by canonical angles,
which can be approximately preserved after JL random projection.
Thus the compressed angle-based subspace visualization algorithm designed as per CSL framework
is likely to yield a very similar result with what original algorithm yields.
We verify this on two real-world data sets,
i.e., YaleB Face data set and Webb Spam Corpus 2006 data set,
with partial Fourier matrices and Bernoulli matrices, respectively.
The experimental results are shown in Figure~\ref{fig:webb-CSV}(a) and (b).
The implementation details and performance analysis are postponed to Appendix \ref{app:CSV}.

According to Figure~\ref{fig:webb-CSV}(a),
the visualization result changes very slightly with the reduced dimension
as long as the compression ratio exceeds $1$E$-1.5$.
This indicates that proper compression does not bring great distortion to the visualization result.
Similar phenomena are observed in Figure~\ref{fig:webb-CSV}(b).
Even when compression ratio is as low as $1$E$-4.5$,
the application of JL random projection is still safe.

\subsection{Compressed Active Subspace Detection}
\label{subsec:CSD}

\begin{figure}[t]
\subfigure[]{
\begin{minipage}{0.48\linewidth}
  \centering
  \includegraphics[width=0.95\textwidth]{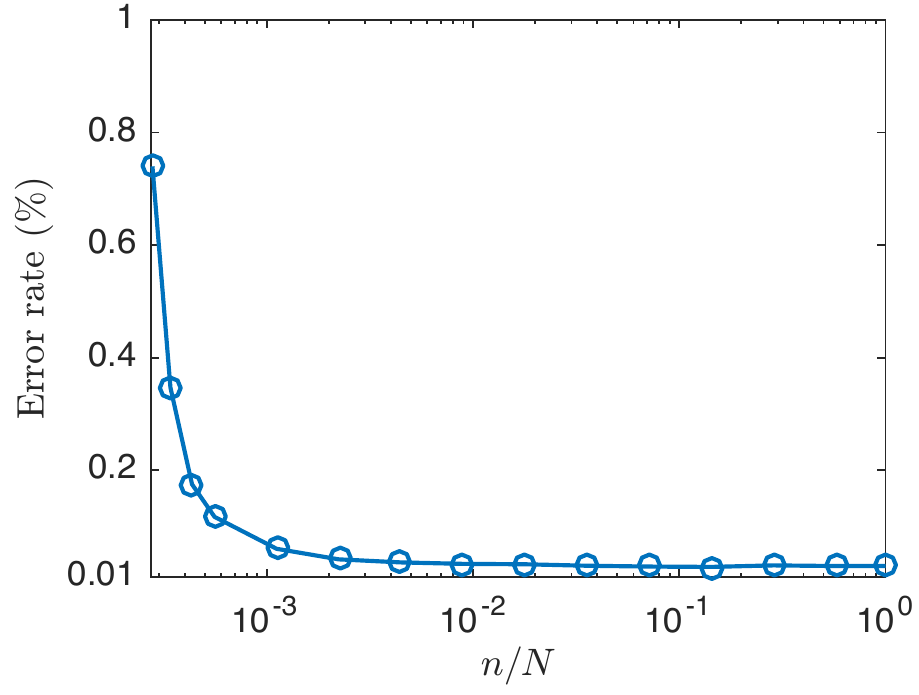}
  \centering
\end{minipage}
}
\hspace{-.15cm}
\subfigure[]{
\begin{minipage}{.48\linewidth}
  \centering
  \includegraphics[width=0.95\textwidth]{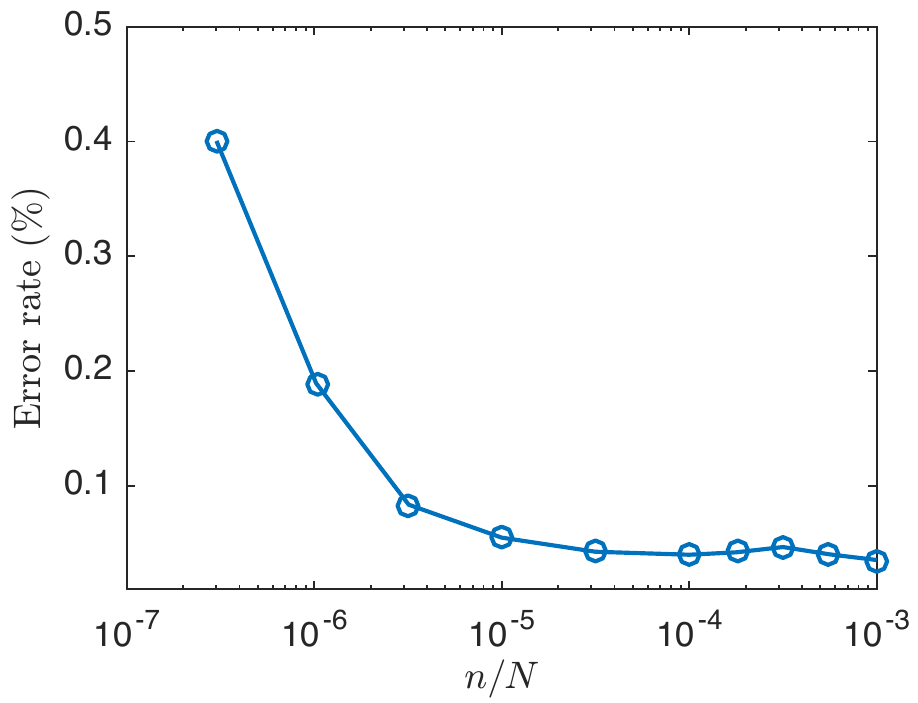}
  \centering
\end{minipage}
}
\caption{The error rate of compressed ML method under different reduced dimensions $n$ on
(a) YaleB Face data set and (b) Webb Spam Corpus 2006 data set.
The results have been averaged over $50$ random trials for (a) and $5$ random trials for (b).}
\label{fig:webb-CSD}
\end{figure}

Active subspace detection refers to identifying which subspace the observed data belongs to
with all candidate subspaces known.
This problem is often encountered in radar target detection,
user detection in wireless network, and image-based verification of employees \citep{Lodhi2017Detection}.
A typical algorithm is the Maximum Likelihood (ML) for active subspace detection,
the performance of which, measured by detection error rate,
is proven to be closely related to canonical angles \citep{Lodhi2017Detection}.
Thus it is likely that the compressed ML method for active subspace detection,
which is designed as per the CSL framework,
can keep almost the same error rate as the ML algorithm without compression.
The implementation details and the performance analysis of the compressed version
are deferred to
Appendix \ref{app:CSD}.

We apply compressed ML for active subspace detection on two real-world data sets,
where YaleB Face data set is compressed by partial Fourier matrices
and Webb Spam Corpus 2006 data set is compressed by Bernoulli matrices.
The experimental results are shown in Figure~\ref{fig:webb-CSD}(a) and (b).
It is observed that the compressed ML for active subspace detection allows for compression ratio as low as $n/N=1$E$-2$ and $1$E$-4$ on these two data sets, respectively.
In other words,
the detection accuracy is kept at a very low level
as long as the compression ratio is higher than this number.
This verifies the performance preserving property of JL random projection.


\subsection{Compressed Subspace Clustering}
\label{subsec:CSC}

\begin{figure}[t]
\subfigure[]{
\begin{minipage}{0.48\linewidth}
  \centering
  \includegraphics[width=0.95\textwidth]{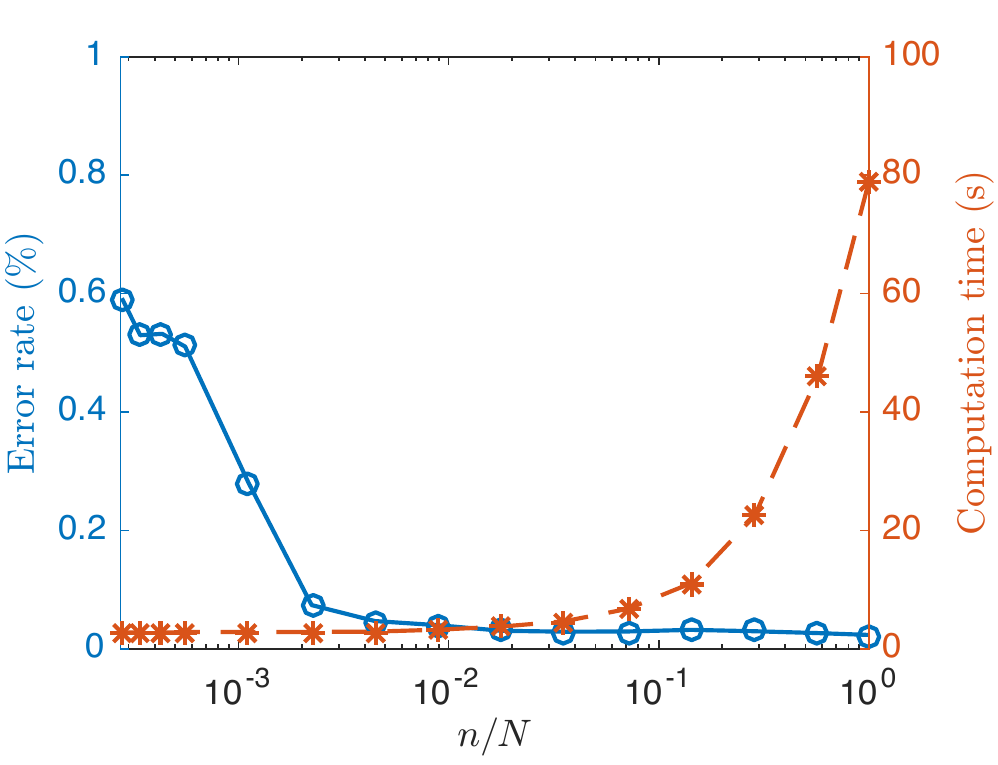}
  \centering
\end{minipage}
}
\subfigure[]{
\begin{minipage}{.47\linewidth}
  \centering
  \includegraphics[width=\textwidth]{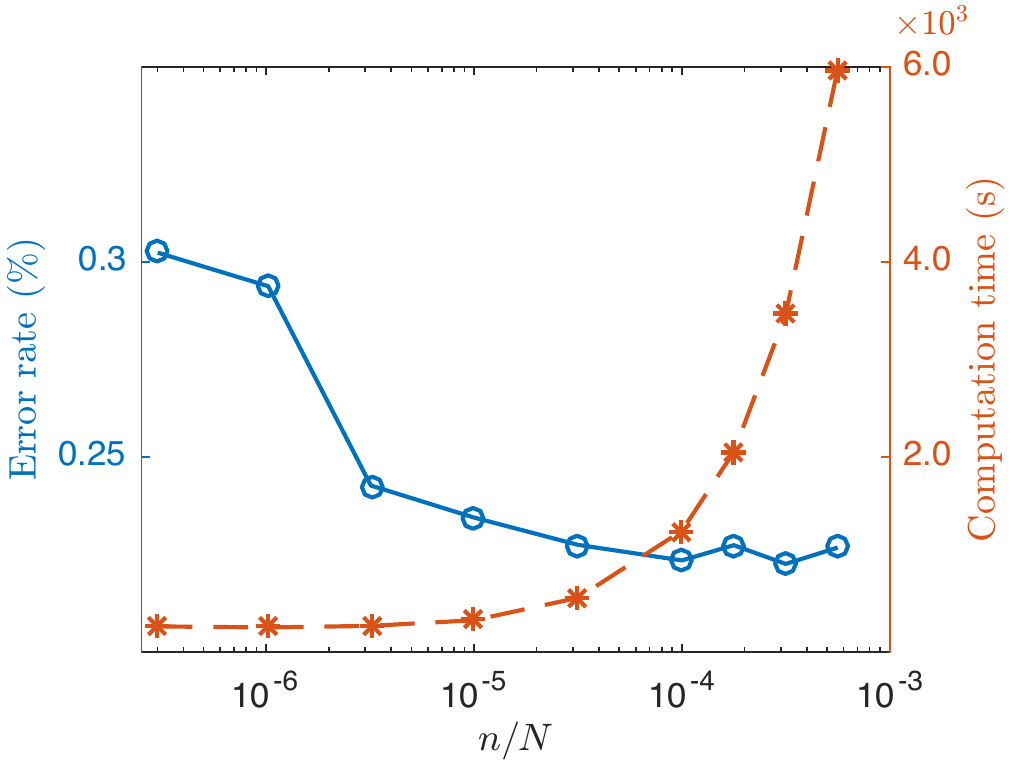}
  \centering
\end{minipage}
  }
   \caption{The error rate and running time of compressed Sparse Subspace Clustering under different reduced dimensions $n$ on (a) YaleB Face data set and (b) Webb Spam Corpus 2006 data set.
The results have been averaged over $10$ random trials for (a) and $5$ random trials for (b).}
   \label{fig:webb-CSC}
\end{figure}

Subspace clustering seeks to find clusters in different subspaces within a data set.
Many algorithms are proposed to solve this problem,
and Sparse Subspace Clustering (SSC) is one of the most popular methods for its high accuracy.
However, this method also undergoes
the high computational complexity when data dimension is high.

To handle this problem,
we design the compressed SSC according to CSL framework,
and test its performance on two real-world data sets.
Again YaleB Face data set and Webb Spam Corpus 2006 data set
are compressed by partial Fourier matrices and Bernoulli matrices, respectively.
The experimental results are shown in Figure~\ref{fig:webb-CSC}(a) and (b).
It is observed that with the decrease of compression ratio,
the running time drops greatly,
while the low error rate is kept.
This demonstrates the power of CSL framework in
reducing the computational complexity without deteriorating the performance
when the original subspace learning algorithms suffers from the curse of dimensionality.

We conclude this section with a final remark that
the subspace clustering algorithm studied here has been widely investigated.
\citet{Mao2014Compressed} and \citet{Wang2019Adeterministic}
adopt
random projections
as dimensionality reduction methods,
and apply subspace clustering on the compressed data set to reduce computational burden.
\citet{Heckel2015dimensionality} analyzes the performance of
SSC, TSC, and SSC-OMP when applied to compressed data set,
under the assumption that the random matrices have RIP for points on union of subspaces.
\citet{Meng2018Ageneral} proposes a general framework capable of
analyzing the performance of various compressed subspace clustering algorithms,
as long as the random matrix has affinity preserving property.
In this work, Theorem~\ref{theorem-rip-subspace} indicates that
JL random projection approximately preserves subspace affinity.
By using this theorem in conjunction with the analysis in \citet{Meng2018Ageneral},
we can provide theoretical performance guarantee for the compressed sparse subspace clustering
algorithm presented in this section.

\section{Conclusion}
\label{sec:conclusion}

In this work, we unveiled the subspace structure preserving property of JL random projection.
Here the subspace structure is described in terms of
canonical angles,
which have the best characterization of relative subspace positions.
Specifically,
it is proved that for a finite collection of $L$ subspaces,
with probability $1-\delta$,
each canonical angle between any two subspaces
is preserved up to $(1\pm\varepsilon)$ when
the dimension is reduced to $n=O(\varepsilon^{-2}\max\{d,\log L,\log(1/\delta)\})$.
This main theoretical result is called CAP property.
Based on this result, we established a general subspace RIP,
which describes the ability to preserve subspace distance of JL random projection.
We say it is \emph{general} because it works for almost arbitrary notion of subspace distance.

Inspired by the above theoretical discovery,
we proposed the CSL framework.
This framework enables to process data lying on UoS
in a space with dimension much lower than the ambient dimension of the data.
This was achieved by safely mapping the data to a space with dimension in the same order of subspace dimensions, which is generally much lower than the ambient dimension,
via JL random projection.
We empirically verified that on subspace clustering algorithms which suffer from the curse of dimensionality,
CSL framework can successfully reduce the time consumption
without deteriorating the performance.
The theoretical foundation of this framework is given by CAP property.
Based on this theory,
we proved the performance preserving property of CSL framework on two other
subspace learning tasks,
namely subspace visualization and active subspace detection.
Considering that our theory is not constrained to specific algorithms,
the extension to other subspace learning tasks is possible.




\appendix

\section{Canonical Angles and Subspace Distances}
\label{app:subspace-distance}

We have emphasized that canonical angles better characterize the relative subspace positions
than projection Frobenius-norm distance,
and any other subspace distances.
We take the following example to support this.
Considering the three two-dimensional subspaces in $\mathbb{R}^4$,
\begin{align}
\mathcal{S}_1:=&{\rm span}\left({\bf e}_1,{\bf e}_2\right),\nonumber\\
\mathcal{S}_2:=&{\rm span}\left({\bf e}_1,{\bf e}_3\right),\nonumber\\
\mathcal{S}_3:=&{\rm span}\left({\bf e}_1+{\bf e}_3,{\bf e}_2+{\bf e}_4\right),\nonumber
\end{align}
where ${\bf e}_i$ for all $i=1,2,3,4$ is the standard basis.
It is obvious that the relative position between $\mathcal{S}_1$ and $\mathcal{S}_2$ differs from that between $\mathcal{S}_1$ and $\mathcal{S}_3$,
for that $\mathcal{S}_1$ and $\mathcal{S}_2$ intersect, while $\mathcal{S}_1$ and $\mathcal{S}_3$ do not.
Such difference can be visualized by
projecting the unit circle in $\mathcal{S}_2$ and $\mathcal{S}_3$ onto $\mathcal{S}_1$,
as shown in Figure~\ref{fig:three_subspaces}.
However, such difference cannot be reflected by projection Frobenius-norm distance $d^{\kappa}$,
which is defined as below.
$$
d^{\kappa}(\mathcal X_1,\mathcal X_2):=\left\|{\bm U}_1{\bm U}_1^{\rm T}-{\bm U}_1{\bm U}_2^{\rm T}\right\|_F^2,
$$
where ${\bm U}_i$ denotes an orthonormal basis for subspace $\mathcal{X}_i$, for $i=1,2$.
According to this definition,
we can verify that the projection Frobenius-norm distance equals to $1$
whether when we measure $\mathcal{S}_1$ and $\mathcal{S}_2$ or $\mathcal{S}_1$ and $\mathcal{S}_3$.
The difference in subspace relative position is not unveiled.
Other notions of subspace distance have similar problems.

\begin{figure}[t]
\vspace{-0.1cm}
 \centering
  \centerline{\includegraphics[width=0.4\textwidth]{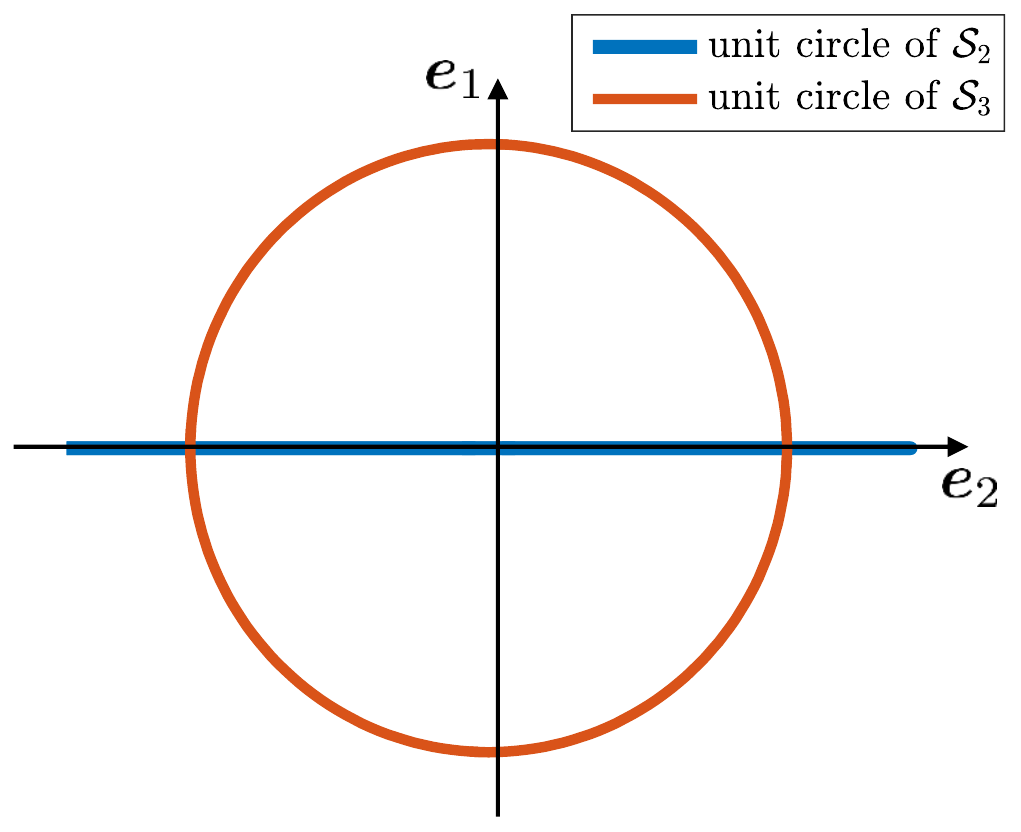}}
\caption{
The projection of unit circles in $\mathcal{S}_2$ and $\mathcal{S}_3$ onto subspace $\mathcal{S}_1$.
The $x$- and $y$-axis represents the projection onto two base vectors of $\mathcal{S}_1$, i.e., ${\bm e}_1$, ${\bm e}_{2}$, respectively.}
\label{fig:three_subspaces}
\end{figure}

The advantage of canonical angles in the ability to describe relative subspace positions
is also shown by the fact that
any notion of rotation-invariant subspace distance is a function of canonical angles.
Rotation-invariant is a natural requirement for distances and is widely satisfied.
We present a list of well-known notions of subspace distance and their dependence on canonical angles in Table~\ref{table-subspace-distance},
where $d$ denotes the dimension of subspaces ${\set X}_1$, ${\set X}_2$,
and $\theta_1\le \cdots\le\theta_d$ denote the canonical angles.
The extension of some notions of distance are listed in Table~\ref{table-subspace-distance-diffd} for subspaces with different dimensions $d_1< d_2$ \citep{YeLim}.

Due to the powerful characterization of canonical angles,
the analyses on the distortion of the above subspace distances
are unified by CAP property.
It is provable that all notions of distance in Table~\ref{table-subspace-distance} and \ref{table-subspace-distance-diffd} are Lipschitz continuous and satisfy \eqref{eq-condition-rip-subspace}.
By Theorem~\ref{theorem-rip-subspace}, all of these distances are approximately preserved by JL random projection.

\begin{table}[t]
\caption{Some definitions of distance between subspaces\label{table-subspace-distance}}
\begin{center}
\begin{tabular}{p{12em}p{22em}}
\hline
\tabincell{c}{projection F-norm} &  {$d_0^{\kappa}\left({\set X}_1, {\set X}_2\right): = \left(\sum_{k=1}^{d} \sin^2\theta_k\right)^{1/2}$}\\
{Fubini-Study} & {$d_0^{\phi}\left({\set X}_1, {\set X}_2\right):=\cos^{-1}\left(\prod_{k=1}^d\cos\theta_k\right)$}\\
{Grassmann} & {$d_0^{\rm G}\left({\set X}_1, {\set X}_2\right): = \left(\sum_{k=1}^{d} \theta_k^2\right)^{1/2}$}\\
{Binet-Cauchy} & {$d_0^{\beta}\left({\set X}_1, {\set X}_2\right):=\left(1-\prod_{k=1}^{d}\cos^2\theta_k\right)^{1/2}$}\\
{Procrustes} & {$d_0^{\rho}\left({\set X}_1, {\set X}_2\right): = 2\left(\sum_{k=1}^{d} \sin^2\left(\theta_k/2\right) \right)^{1/2}$} \\
{Asimov} & {$d_0^{\alpha}\left({\set X}_1, {\set X}_2\right):=\theta_d$}\\
{Spectral} & {$d_0^{\sigma}\left({\set X}_1, {\set X}_2\right):=2\sin\left(\theta_d/2\right)$}\\
{Projection} & {$d_0^{\pi}\left({\set X}_1, {\set X}_2\right):=\sin\theta_d$}\\
\hline
\end{tabular}
\end{center}
\label{tab:distance-sameDim}
\end{table}%

\begin{table}[t]
\caption{Some generalized definitions of distance on subspace with different dimensions \label{table-subspace-distance-diffd}}
\begin{center}
\begin{tabular}{p{12em}p{25em}}
\hline
 \tabincell{c}{projection F-norm}
 &$d^{\kappa}\left({\set X}_1, {\set X}_2\right): = \left((d_2 - d_1)/2 + \sum_{k=1}^{d_1} \sin^2\theta_k\right)^{1/2}$\\
Grassmann & $d^{\rm G}\left({\set X}_1, {\set X}_2\right): = \left( \left(d_2 - d_1\right)\pi^2/4 + \sum_{k=1}^{d_1} \theta_k^2\right)^{1/2}$\\
Procrustes & $d^{\rho}\left({\set X}_1, {\set X}_2\right): = \left(d_2 - d_1 + 2\sum_{k=1}^{d_1} \sin^2\left(\theta_k/2\right) \right)^{1/2}$\\
\hline
\end{tabular}
\end{center}
\label{default}
\end{table}%

It is sometimes useful to have different notions of subspace distance
for specific applications.
The projection Frobenius-norm distance is also called chordal distance.
It is widely used in the subspace quantization problem appearing
in the precoding of multiple-antenna wireless systems \citep{Limited2005Love}.
The Fubini-Study distance is also investigated in this problem \citep{Limitedfeedback2005Love}.
The Grassmann distance is the geodesic distance when viewing subspaces as points on the Grassmannian manifold,
i.e., it can be locally interpreted as the shortest length of all curves between the two measured subspaces on manifold.
It can be used to assess the convergence of the Riemann-Newton method \citep{Riemannian2004Absil}.
Both Binet-Cauchy and the projection Frobenius-norm distance are used in
some subspace learning algorithms for the positive definiteness of the kernel they are induced from \citep{Grassmann2008Hamm}.
Definitions determined by single canonical angle, such as the Asimov distance, the spectral distance,
and the projection distance,
share the advantage of being more robust to noise.


\section{Supplement for the Proof of Theorem~\ref{theorem-angle-multi}}
\label{app:proof-cosine-preserving}

To establish $n=O(\varepsilon^{-2})$ for $\theta\in[\pi/4,\pi/2]$,
we will first prove
\begin{align}\label{eq-cosine-preserving}
\left|\cos^2\psi_k-\cos^2\theta_k\right|\le C_1\varepsilon\cos\theta_k+C_2\varepsilon^2,
\end{align}
where $C_1$, $C_2$ are universal constants,
and then derive the relationship between $n$ and $\varepsilon$.

To prove \eqref{eq-cosine-preserving},
according to Lemma~\ref{lemma-step1},
it suffices to prove
\begin{align*}
\mathbb{P}\left(\cos^2\varphi_{1,k,k}- \cos^2\theta_k\le C_1\varepsilon\cos\theta_k+C_2\varepsilon^2 \right)&\ge 1-{\rm e}^{-c_2\varepsilon^2n},\\
\mathbb{P}\left( \cos^2\theta_k-\cos^2\varphi_{k,d_1,1}\le C_1\varepsilon\cos\theta_k+C_2\varepsilon^2\right)&\ge 1-{\rm e}^{-c_2\varepsilon^2n}.
\end{align*}
Proofs of these two inequalities are similar with
the argument in Section~\ref{subsubsec:proof-eq-step21}
and \ref{subsubsec:proof-eq-step22}.
Readers only need to replace Lemma~\ref{lemma-line} with the following lemma implied in \citet{xingyu}.
\begin{lemma}\label{lemma-line-2}
Suppose ${\bm \Phi}$ is a random matrix with JL property.
Suppose ${\bm x}_1$, $\set X_2$ are respectively a vector and a $d$-dimensional
subspace in $\mathbb R^N$.
Denote ${\bm y}_1:={\bm \Phi}{\bm x}_1/\left\|{\bm \Phi}{\bm x}_1\right\|$ and $\set Y_2$ as
the projection of $\set X_2$ with ${\bm \Phi}$.
Then there exist positive universal constants $c_1$, $c_2$,
for any $\varepsilon\in(0,1/2)$ and any $n>c_1\varepsilon^{-2}d$, we have
$$
    \left|\left\|{\bm P}_{\mathcal{Y}_2}({\bm y}_1)\right\|^2-\left\|{\bm P}_{\mathcal{X}_2}({\bm x}_1)\right\|^2\right|\le C_1\varepsilon\left\|{\bm P}_{\mathcal{X}_2}({\bm x}_1)\right\|+C_2\varepsilon^2
$$
with probability at least $1-{\rm e}^{-c_2\varepsilon^2n}$.
\end{lemma}
\begin{proof}
The proof follows from \citet{xingyu}, Lemma 2.
\end{proof}

According to \eqref{eq-cosine-preserving},
when $\cos\theta_k\le\varepsilon$, we could immediately get
\begin{align}\label{eq-temp-25}
\left|\cos\psi_k-\cos\theta_k\right|\le C_0 \varepsilon
\end{align}
with probability at least $1-{\rm e}^{-c_2\varepsilon^2n}$,
where $C_0$ is a positive universal constant.
When $\cos\theta_k>\varepsilon$, we have
\begin{align}\label{eq-temp-26}
\left|\cos\psi_k-\cos\theta_k\right|= \frac{\left|\cos^2\psi_k-\cos^2\theta_k\right|}{\cos\psi_k+\cos\theta_k}\le C_1\varepsilon+C_2\varepsilon.
\end{align}
From \eqref{eq-temp-25} and \eqref{eq-temp-26}, recalling that $\theta\in[\pi/4,\pi/2]$,
we could prove
$$
\psi_k-\theta_k\le C \varepsilon\theta_k,
$$
where $C$ is a positive universal constant and thus $n=O(\varepsilon^{-2})$.


\section{Supplement for Section~\ref{sec:app-compression}}
\label{sec:theory-app}

In this section,
we will introduce implementation details
of two compressed algorithms,
namely compressed angle-based subspace visualization and
compressed ML for active subspace detection,
and theoretically analyze their performance distortion.
Throughout this section, we denote the number of samples as $M$.

\subsection{Compressed Subspace Visualization}
\label{app:CSV}

We first review the angle-based subspace visualization algorithm proposed in \citet{Shen2018Subspace}.
It takes data points ${\bm x}_1,\cdots,{\bm x}_M$ as input,
which lie on $L$ subspaces $\mathcal{X}_1,\cdots,\mathcal{X}_L$.
The labels of data and the bases of $L$ subspaces are assumed to be known.
The algorithm first constructs a dissimilarity matrix ${\bm D}$,
and then embeds the data into a $2$- or $3$-dimensional space via MDS.
The dissimilarity matrix ${\bm D}:=(D_{i,j})_{i,j}$ is defined as below.
\begin{equation}\label{def-D-ij}
D_{\rm SV}({\bm x}_i,{\bm x}_j) = \left\{ \begin{array}{ll}
\sin^2\theta^{i,j}, & {\bm x}_i, {\bm x}_j {\rm ~lie~in~the~same~subspace};\\
\left(v\sin\theta^{i,j} +u\min_k(\sin\tilde{\theta}^{i,k} +\sin\tilde{\theta}^{j,k}) \right)^2, & {\rm otherwise};
\end{array}
\right.
\end{equation}
where $\theta^{i,j}$ denotes the canonical angle between ${\bm x}_i$ and ${\bm x}_j$,
and $\tilde{\theta}^{i,k}$ denotes the canonical angle between ${\bm x}_i$ and $\mathcal{X}_k$.
$u$,$v$ denote two algorithmic parameters, which balance the term $\sin\theta^{i,j}$ and $\min_k(\sin\tilde{\theta}^{i,k} +\sin\tilde{\theta}^{j,k})$.
The second step MDS is completed by applying eigenvalue decomposition on double-centered distance matrix
\begin{align}\label{def-B}
{\bm B}:=-\frac12{\bm H}{\bm D}{\bm H},
\end{align}
and obtain eigenvectors ${\bm v}_1,{\bm v}_2$ (or ${\bm v}_1,{\bm v}_2, {\bm v}_3$) corresponding to the largest two (or three) eigenvalues $\lambda_1\ge\lambda_2$ (or $\lambda_1\ge\lambda_2\ge\lambda_3$),
where ${\bm H}:={\bm I}-1/M{\bm 1}{\bm 1}^{\rm T}$ is the centering matrix and ${\bm 1}\in\mathbb{R}^M$ denotes all-ones vector.
The output coordinate matrix is denoted by ${\bm C}=[\lambda_1^{1/2}{\bm v}_1,\lambda_2^{1/2}{\bm v}_2]$ (or ${\bm C}=[\lambda_1^{1/2}{\bm v}_1,\lambda_2^{1/2}{\bm v}_2, \lambda_3^{1/2}{\bm v}_3]$).

According to Algorithm \ref{alg:GRPLCA},
we design compressed angle-based subspace visualization algorithm as below.
\begin{definition}\label{def-CSV}
There are three steps to implement compressed angle-based subspace visualization algorithm.
1. Projecting data ${\bm x}_i\in\mathbb{R}^N$ with a partial Fourier matrix and getting the projected data ${\bm y}_i$.
2. Calculating dissimilarity matrix $\hat{\bm D}$ by using ${\bm y}_i$.
3. Applying MDS and returning output coordinate matrix $\hat{\bm C}$.
\end{definition}

The following theorem states the error introduced by replacing ${{\bm D}}$ with $\hat{{\bm D}}$.
\begin{theorem}\label{theorem-CSV}
Suppose there are $M$ data points in $\mathbb{R}^N$, lying on $L$ subspaces with dimensions no more than $d$.
The dissimilarity matrix constructed according to \eqref{def-D-ij} is denoted as ${\bm D}\in\mathbb{R}^{M\times M}$.
The output coordinate matrix is denoted as ${\bm C}\in\mathbb{R}^{M\times 2}$.
After applying random projection with partial Fourier matrices,
we get $M$ compressed data points in $\mathbb{R}^n$.
The corresponding dissimilarity matrix and coordinate matrix is denoted as $\hat{{\bm D}}$ and $\hat{{\bm C}}$, respectively.
Assume the eigenvalue of double-centered matrix ${\bm B}$ defined in \eqref{def-B} satisfies $\lambda_1>\lambda_2>\lambda_3\ge\cdots\ge\lambda_M$.%
\footnote{The condition that $\lambda_1\neq\lambda_2\neq\lambda_3$ is necessary.
The reason is that if there exist repeated eigenvalues,
the visualization result is not unique.
In this case,
measuring the distortion caused by JL random projection becomes ill-defined.}
For any $\varepsilon\in(0,1)$, there exist two positive universal constants $c_1$, $c_2$,
such that for any $n\ge c_1\varepsilon^{-2}\max\{d,\log(M+L)\}$, with probability at least $1-{\rm e}^{-c_2\varepsilon^2n}$, we have
\begin{align}\label{eq-theorem-CSV}
\left\|\hat{{\bm C}}-{\bm C}\right\|_F^2\le2\left\|{\bm D}\right\|_F\varepsilon+\frac{4\left\|{\bm D}\right\|_F^2\lambda_1}{(\lambda_1-\lambda_2)^2}\varepsilon^2+\frac{4\left\|{\bm D}\right\|_F^2\lambda_2}{\min\{(\lambda_1-\lambda_2)^2,(\lambda_2-\lambda_3)^2\}}\varepsilon^2.
\end{align}
\end{theorem}
\begin{proof}
The proof is postponed to Appendix \ref{proof-theorem-CSV}.
\end{proof}
By \eqref{eq-theorem-CSV}, we have
$
\|\hat{{\bm C}}-{\bm C}\|_F^2\le2\|{\bm D}\|_F\varepsilon+o(\varepsilon).
$
When $\varepsilon$ is small, the error in visualization caused by JL random projection is also small.

\begin{remark}
The case considered in Theorem~\ref{theorem-CSV} is that data points are visualized in a two dimensional plot.
When we visualize data points in a three dimensional plot,
i.e., take the eigenvectors of ${\bm B}$ corresponding to the largest three eigenvalues in MDS step,
the visualization error will further increase by
$$
\left\|{\bm D}\right\|_F\varepsilon+\frac{4\left\|{\bm D}\right\|_F^2\lambda_3}{\min\{(\lambda_2-\lambda_3)^2,(\lambda_3-\lambda_4)^2\}}\varepsilon^2,
$$
under the assumption that $\lambda_3>\lambda_4$.
\end{remark}

\subsection{Compressed Active Subspace Detection}
\label{app:CSD}

Active subspace detection can be mathematically written as the following hypothesis problem.
\begin{align*}
\mathcal{H}_i:\quad {\bm x}={\bm U}_i{\bm s}+{\bm n},
\end{align*}
where ${\bm U}_i$ denotes the orthonormal basis for the $i$-th subspace $\mathcal{X}_i$
with dimension $d$,
and ${\bf n}$ denotes the additive Gaussian white noise.
Denote $\mathbb{P}_{\mathcal{H}_i}(\cdot)$ as the probability conditioned on hypothesis $\mathcal{H}_i$,
and $\bar{\mathcal{H}}_i$ as the event that hypothesis $\mathcal{H}_i$ is accepted.
Assume that the a prior probability of each hypothesis is the same.
Then $\frac1L\sum_{i=1}^L\left(1-\mathbb{P}_{\mathcal{H}_i}(\bar{\mathcal{H}}_i)\right)$ is defined as the error rate,
which is what we are interested in.

The Maximum Likelihood (ML) method for active subspace detection we are concerned follows that
\begin{align}\label{eq-SD-ML}
\bar{i}:=\arg\max_{1\le i\le L} \left\|{\bm U}_{i}^{\rm T}{\bm x}\right\|.
\end{align}
To analyze the performance of the detector given by \eqref{eq-SD-ML},
we first give the definition of affinity in terms of canonical angles,
and then show how affinity influences the detection error rate.
\begin{definition}\citep{soltanolkotabi2012geometric}\label{def-aff}
The affinity ${\rm aff}(\mathcal X_1,\mathcal X_2)$ between subspaces $\mathcal X_1$ and $\mathcal X_2$ with dimension $d$ is defined as below.
$$
{\rm aff}^2(\mathcal X_1,\mathcal X_2) = \sum_{k=1}^{d}\cos^2\theta_k,
$$
where $\theta_k$ denotes the $k$-th canonical angle.
\end{definition}

For ease of use, we assume the covariance matrix of noise to be any positive definite matrix,
and analyze the performance of detector \eqref{eq-SD-ML}.
\begin{lemma}\label{lemma-ML}
Assume that ${\bm s}$ follows Gaussian distribution $\mathcal{N}({\bm 0}, 1/d{\bm I})$,
and noise ${\bm n}$ follows Gaussian distribution $\mathcal{N}({\bm 0}, 1/d{\bm R}_n)$.
Denote the maximum eigenvalue of ${\bm R}_n$ as $\delta$.
Denote the affinity between subspace $\mathcal{X}_i$ and $\mathcal{X}_j$ as ${\rm aff}_{ij}$.
Then the correct probability of the detector \eqref{eq-SD-ML} is given by
\begin{align}\label{eq-acc-SD-ML}
\mathbb{P}_{\mathcal{H}_i}(\bar{\mathcal{H}}_i)&\ge
1-4\sum_{j\neq i}{\rm e}^{-C\left({\rm aff}_{ij},\delta\right)d},
\end{align}
where
\begin{align}\label{eq-ML-C}
C\left({\rm aff}_{ij},\delta\right):=\frac18\cdot \frac{\left(\left(1-{\rm aff}_{ij}^2/d-\delta\right)/\left(1+\delta\right)-8/d\right)^2}{4+\left(1-{\rm aff}_{ij}^2/d-\delta\right)/\left(1+\delta\right)-8/d}.
\end{align}
\end{lemma}
\begin{proof}
The proof is postponed to Appendix \ref{sec:proof-SD-ML}.
\end{proof}

We design the compressed ML method for active subspace detection following Algorithm~\ref{alg:GRPLCA} as below.
\begin{definition}\label{def-CSD}
There are two steps to implement compressed ML method for active subspace detection.
1.~Projecting data ${\bm x}\in\mathbb{R}^N$ with partial Fourier matrices and getting the compressed data ${\bm y}$.
2.~Calculating $\bar{i}:=\arg\max_{1\le i\le L} \left\|{\bm V}_{i}^{\rm T}{\bm y}\right\|$,
where ${\bm V}_i$ denotes the orthonormal basis of the $i$-th compressed subspace.
\end{definition}

What follows is the performance analysis of compressed ML for active subspace detection.

\begin{theorem}\label{corollary-CSD-ML}
Under the same setting as Lemma~\ref{lemma-ML},
when we apply compressed ML method for active subspace detection,
for any $\varepsilon\in(0,1/2)$, there exist positive universal constants $c_1$, $c_2$,
such that for any $n>c_1\varepsilon^{-2}\max\{d, \log L\}$,
the correct probability satisfies
\begin{align}\label{eq-CSD-error}
\mathbb{P}_{\mathcal{H}_i}(\bar{\mathcal{H}}_i)\ge1-4\sum_{j\neq i}{\rm e}^{-C(\hat{{\rm aff}}_{ij},\hat{\delta})d}-{\rm e}^{-c_2\varepsilon^2n},
\end{align}
where $C(\cdot,\cdot)$ is defined in \eqref{eq-ML-C}, and
\begin{align*}
\hat{{\rm aff}}_{ij}^2:&={\rm aff}_{ij}^2(1+\varepsilon),\\
\hat{\delta}:&=\delta(1+\varepsilon)^2+\varepsilon^2.
\end{align*}
\end{theorem}
\begin{proof}
After projection, we have
${\bm y}={\bm \Phi}{\bm x}={\bm \Phi}{\bm U}_i{\bm s}+{\bm \Phi}{\bm n}$.
The singular decomposition gives ${\bm \Phi}{\bm U}_i={\bm V}_i{\bm \Lambda}_i{\bm Q}_i^{\rm T}$,
then we have
$$
{\bm y}={\bm V}_i{\bm \Lambda}_i{\bm Q}_i^{\rm T}{\bm s}={\bm V}_i{\bm Q}_i^{\rm T}{\bm s}+{\bm V}_i({\bm \Lambda}_i-{\bm I}){\bm Q}_i^{\rm T}{\bm s}.
$$
Denote $\bar{{\bm y}}:={\bm V}_i{\bm Q}_i^{\rm T}{\bm s}$ and  ${\bm w}:={\bm V}_i({\bm \Lambda}_i-{\bm I}){\bm Q}_i^{\rm T}{\bm s}+{\bm \Phi}{\bm n}$, we have ${\bm y}=\bar{\bm y}+{\bm w}$.
In the first item below, we will show that
$\bar{\bm y}$ can be regarded as a signal uniformly distributed within the $i$-th projected subspace $\mathcal{Y}_i$.
While in the second item, we will show that with probability at least $1-{\rm e}^{-c\varepsilon^2n}$, ${\bm w}$ can be regarded as Gaussian noise whose covariance matrix satisfying $\lambda_{\max}({\bm U}_i^{\rm T}{\bm R}_w{\bm U}_i)\le\delta(1+\varepsilon)^2+\varepsilon^2$.

1)
Noticing that ${\bm V}_i$ is an orthonormal basis for $\mathcal{Y}_i$,
we have $\bar{\bm y}\in\mathcal{Y}_i$.
Its projection onto $\mathcal{Y}_i$ is ${\bm V}_i^{\rm T}\bar{\bm y}={\bm Q}_i^{\rm T}{\bm s}\sim\mathcal{N}({\bm 0}, 1/d_i{\bm I})$.

2)
The covariance matrix of noise ${\bm n}$ is $1/d{\bm V}_i({\bm \Lambda}_i-{\bm I})^2{\bm V}_i^{\rm T}+1/d{\bm \Phi}^{\rm T}{\bm R}_n{\bm \Phi}$.
According to Lemma \ref{lemma:orthonormal-preserving}, and noticing that the diagonal entries of ${\bm \Lambda}_i$ are singular values of matrix ${\bm \Phi}{\bm U}_i$,
we have $\lambda_{\max}({\bm U}_i^{\rm T}{\bm R}_w{\bm U}_i)\le\varepsilon^2+\delta(1+\varepsilon)^2$ with probability at
least $1-{\rm e}^{-c_2\varepsilon^2 n}$ for any $n>c_1\varepsilon^{-2}\max\{d,\log L\}$,
where $c_1$, $c_2$ are two positive universal constants.
Then the compressed setting is similar to the noisy setting, whose correct probability is shown in \eqref{eq-acc-SD-ML}.
According to Theorem \ref{theorem-angle-multi}, canonical angles between $\mathcal{Y}_i$, $\mathcal{Y}_j$ satisfy $\psi_k^{i,j}\ge(1-\varepsilon)\theta_k^{i,j}$ with overwhelming probability.
Replacing $\theta_k^{i,j}$ with $(1-\varepsilon)\theta_k^{i,j}$ in \eqref{eq-acc-SD-ML}, we complete the proof.
\end{proof}

Comparing \eqref{eq-acc-SD-ML} and \eqref{eq-CSD-error},
it is clear that as long as the reduced dimension $n$ satisfies $n>c_1\varepsilon^{-2}\max\{d,\log L\}$,
JL random projection will not bring great degradation to the correct probability $\mathbb{P}_{\mathcal{H}_i}(\bar{\mathcal{H}}_i)$
for $i=1,\cdots,L$.

\section{Proof of Theorem~\ref{theorem-CSV}}
\label{proof-theorem-CSV}

Denote the eigenvectors of double-centered dissimilarity matrix ${\bm B}$ defined in \eqref{def-B} as ${\bm v}_1,\cdots,$ ${\bm v}_M$, corresponding to eigenvalues $\lambda_{1}>\lambda_2\ge\lambda_3\ge\cdots\ge\lambda_{M}$.
Denote the first two eigenvectors and eigenvalues of matrix $\hat{{\bm B}}:=-1/2{\bm H}\hat{{\bm D}}{\bm H}$ as $\hat{{\bm v}}_1$, $\hat{{\bm v}}_2$ and $\hat{\lambda}_1$, $\hat{\lambda}_2$, respectively.
Then we have
\begin{align}
\left\|{\bm C}-\hat{\bm C}\right\|_F^2=&
\sum_{i=1}^2\left\|\lambda_i^{1/2}{\bm v}_i-\hat{\lambda}_i^{1/2}\hat{{\bm v}}_i\right\|^2\nonumber\\
=&\sum_{i=1}^2\left\|\lambda_i^{1/2}{\bm v}_i-\lambda_i^{1/2}\hat{{\bm v}}_i+\lambda_i^{1/2}\hat{{\bm v}}_i-\hat{\lambda}_i^{1/2}\hat{{\bm v}}_i\right\|^2\nonumber\\
\le &\sum_{i=1}^2 \left(2\left\|\lambda_i^{1/2}{\bm v}_i-\lambda_i^{1/2}\hat{{\bm v}}_i\right\|^2+2\left(\lambda_i^{1/2}-\hat{\lambda}_i^{1/2}\right)^2\right)\nonumber\\
\le&\sum_{i=1}^2 \left(2\lambda_i\left\|{\bm v}_i-\hat{{\bm v}}_i\right\|^2+2\left|\lambda_i-\hat{\lambda}_i\right|\right).\label{eq-temp-8}
\end{align}
Next we need to use the following lemma about the perturbation theory of matrix eigenvalues and eigenvectors.

\begin{lemma}\label{lemma-perturbation-eigenvalue}(simplified from \citet{Yu2014auseful}, Theorem 2)
Assume matrices ${\bm A},\hat{{\bm A}}\in\mathbb{R}^{n\times n}$ are symmetric, with eigenvalues $\lambda_1\ge\cdots\ge\lambda_n$ and $\hat{\lambda}_1\ge\cdots\ge\hat{\lambda}_n$, respectively.
Fix $1\le i\le n$ and assume that $\lambda_{i-1}>\lambda_i>\lambda_{i+1}$,
where $\lambda_0:=\infty$ and $\lambda_{n+1}:=-\infty$.
Let ${\bm v}_i$ and $\hat{\bm v}_i$ be the vectors satisfying ${\bm A}{\bm v}_i=\lambda_i{\bm v}_i$ and
$\hat{{\bm A}}\hat{{\bm v}}_i=\hat{\lambda}_i\hat{{\bm v}}_i$, respectively,
and ${\bm v}_i^{\rm T}\hat{{\bm v}}_i\ge0$.
Then
$$
\sin\langle{\bm v}_i,\hat{{\bm v}}_i\rangle\le\frac{2\left\|\hat{{\bm A}}-{\bm A}\right\|_F}{\min\{\lambda_{i-1}-\lambda_i,\lambda_i-\lambda_{i+1}\}}
$$
and
$$
\left|\hat{\lambda}_i-\lambda_i\right|\le\left\|\hat{{\bm A}}-{\bm A}\right\|_F,
$$
where $\langle{\bm v}_i,\hat{{\bm v}}_i\rangle$ denotes the angle between ${\bm v}_i$ and $\hat{{\bm v}}_i$.
\end{lemma}

Denote
$$
{\bm E}:=\hat{{\bm B}}-{\bm B}=-\frac12{\bm H}\left(\hat{{\bm D}}-{\bm D}\right){\bm H}
$$
as the error of centered dissimilarity matrix caused by random projection.
According to Lemma~\ref{lemma-perturbation-eigenvalue},
we have for $i=1,2$,
\begin{align}
\left\|{\bm v}_i-\hat{{\bm v}}_i\right\|^2&\le 2\sin^2\langle{\bm v}_i,\hat{{\bm v}}_i\rangle
\le
\frac{8\left\|{\bm E}\right\|_F^2}{\min\{(\lambda_{i-1}-\lambda_i)^2,(\lambda_i-\lambda_{i+1})^2\}},\label{eq-temp-6}\\
\left|\hat{\lambda}_i-\lambda_i\right|&\le\left\|{\bm E}\right\|_F.\label{eq-temp-7}
\end{align}

Now we further simplify the F-norm of ${\bm E}$.
Considering that matrix ${\bm H}$ is a projection matrix with spectral norm no more than $1$,
we have
$$
\left\|{\bm E}\right\|_F\le\frac12\left\|\hat{{\bm D}}-{\bm D}\right\|_F.
$$
According to Lemma~\ref{lemma-sin}, for any $n\ge c_1\varepsilon^{-2}\max\{d,\log(M+L)\}$, with probability at least $1-{\rm e}^{-c_2\varepsilon^2n}$,
we have
$$
\big\|\hat{{\bm D}}-{\bm D}\big\|_F\le\left\|{\bm D}\right\|_F\varepsilon,
$$
 and thus
\begin{align}\label{eq-temp-5}
\left\|{\bm E}\right\|_F\le\frac12\left\|{\bm D}\right\|_F\varepsilon.
\end{align}

Plugging \eqref{eq-temp-6}, \eqref{eq-temp-7}, and \eqref{eq-temp-5} into \eqref{eq-temp-8},
we get \eqref{eq-theorem-CSV} and complete the proof.

\section{Proof of Lemma~\ref{lemma-ML}}
\label{sec:proof-SD-ML}

For any $j\neq i$, denote the principal vectors in $\mathcal{X}_j$ for the subspaces pair $\mathcal{X}_i$ and $\mathcal{X}_j$ as $\tilde{{\bm U}}_j$.
According to the process of ML method, we have
\begin{align*}
1-\mathbb{P}_{\mathcal{H}_i}(\bar{\mathcal{H}}_i)&\le \sum_{j\neq i}\mathbb{P}\left(\left\|\tilde{{\bm U}}_i^{\rm T}{\bm U}_i{\bm s}+\tilde{{\bm U}}_i^{\rm T}{\bm n}\right\|^2-\left\|\tilde{{\bm U}}_j^{\rm T}{\bm U}_i{\bm s}+\tilde{{\bm U}}_j^{\rm T}{\bm n}\right\|^2<0\right)\nonumber\\
&=\sum_{j\neq i}\mathbb{P}\left(\left\|\tilde{{\bm s}}+\tilde{{\bm U}}_i^{\rm T}{\bm n}\right\|^2-\left\|{\bm \Lambda}_{ij}\tilde{{\bm s}}+\tilde{{\bm U}}_j^{\rm T}{\bm n}\right\|^2<0\right),
\end{align*}
where $\tilde{\bm s}:=\tilde{\bm U}_i^{\rm T}{\bm U}_i{\bm s}$ is the rotation of ${\bm s}$ and it still follows distribution $\mathcal{N}({\bm 0},1/d{\bm I})$.
Now it suffices to prove the following inequalities
\begin{align*}
\mathbb{P}\left(\left\|\tilde{{\bm s}}+\tilde{{\bm U}}_i^{\rm T}{\bm n}\right\|^2-1<-\tau\right)
&<2{\rm e}^{-\tilde{C}(\tau,\delta)d},\\
\mathbb{P}\left(\left\|{\bm \Lambda}_{ij}\tilde{{\bm s}}+\tilde{{\bm U}}_j^{\rm T}{\bm n}\right\|^2-\left(\frac{{\rm aff}_{ij}^2}{d}+\delta\right)>\tau\right)
&<2{\rm e}^{-\tilde{C}(\tau,\delta)d},
\end{align*}
where
\begin{align*}
\tilde{C}(\tau,\delta):&=\frac{\left(\tau/(1+\delta)-4/d\right)^2}{4\left(2+\tau/\left(1+\delta\right)-4/d\right)},\\
\tau:&=\frac12\left(1-\frac{{\rm aff}_{ij}^2}{d}-\delta\right).
\end{align*}
This can be immediately obtained from the following lemma by noticing that
$\sigma_1({\bm I}+{\bm R}_n)\le1+\delta$.
\begin{lemma}
Assume ${\bm a}\in\mathbb{R}^d$ and its $i$-th entry $a_i{\overset{i.i.d.}{\sim}}\mathcal{N}(0,1/d)$.
Then for $0\le\lambda_1\le\cdots\le\lambda_d\le 1$,
for any $\tau>0$,
we have
$$
\mathbb{P}\left(\left|\sum_{i=1}^d\lambda_i^2a_i^2-\frac{\sum_{i=1}^d\lambda_i^2}{d}\right|>\tau\right)
<2\exp\left(\frac{d}{4}\cdot\frac{(\tau-4/d)^2}{2+\tau-4/d}\right).
$$
\end{lemma}
\begin{proof}
Let
\begin{equation*}
f({\bm a}) = \sqrt{\sum\nolimits_{i=1}^d \lambda_i^2 a_i^2},
\end{equation*}
then by calculation
\begin{equation*}
\|\nabla f({\bm a})\| = \sqrt{\frac{\sum_{i=1}^d \lambda_i^4 a_i^2}{\sum_{i=1}^d \lambda_i^2 a_i^2}} \le 1.
\end{equation*}
Hence, $f({\bm a})$ is $1$-Lipschitz and according to the concentration in Gauss space
(\citep{ledoux2001concentration}, or see Proposition 5.34 in \citet{Vershynin2010Introduction}), we have
\begin{equation*}
\mathbb{P}\left(\sqrt{\sum\nolimits_{i=1}^d \lambda_i^2 a_i^2} - \mathbb{E} \sqrt{\sum\nolimits_{i=1}^d \lambda_i^2 a_i^2} > t \right) < {\rm e}^{-\frac{dt^2}{2}}.
\end{equation*}
Take $f(x) = -\sqrt{\sum_{i=1}^d \lambda_i^2 x_i^2}$, then similarly
\begin{equation*}
\mathbb{P}\left(\sqrt{\sum\nolimits_{i=1}^d \lambda_i^2 a_i^2} - \mathbb{E} \sqrt{\sum\nolimits_{i=1}^d \lambda_i^2 a_i^2} < -t \right) < {\rm e}^{-\frac{dt^2}{2}}.
\end{equation*}
Moreover, $\left(\mathbb{E} \sqrt{\sum_{i=1}^d \lambda_i^2 a_i^2}\right)^2 \le \mathbb{E} \sum_{i=1}^d \lambda_i^2 a_i^2 = \frac{\sum_{i=1}^d \lambda_i^2}{d}$ and
\begin{align*}
&\left(\mathbb{E} \sqrt{\sum\nolimits_{i=1}^d \lambda_i^2 a_i^2}\right)^2
\\=& \mathbb{E} \sum_{i=1}^d \lambda_i^2 a_i^2 - {\rm Var}\left(\sqrt{\sum\nolimits _{i=1}^d \lambda_i^2 a_i^2}\right) \\
=& \frac{\sum_{i=1}^d \lambda_i^2}{d} - \int_t t^2 {\rm d}\mathbb{P}\left(\left|\sqrt{\sum\nolimits_{i=1}^d \lambda_i^2 a_i^2} - \mathbb{E} \sqrt{\sum\nolimits_{i=1}^d \lambda_i^2 a_i^2}\right| < t \right) \\
=& \frac{\sum_{i=1}^d \lambda_i^2}{d} - \int_t 2t \mathbb{P}\left(\left|\sqrt{\sum\nolimits_{i=1}^d \lambda_i^2 a_i^2} - \mathbb{E} \sqrt{\sum\nolimits_{i=1}^d \lambda_i^2 a_i^2}\right| > t \right) {\rm d}t \\
\ge& \frac{\sum_{i=1}^d \lambda_i^2 - 4}{d}.
\end{align*}
Taking square, we have
$$
\mathbb{P}\left(\left|\sum_{i=1}^d \lambda_i^2 a_i^2 - \frac{\sum_i \lambda_i^2}{d}\right| > 2t\sqrt{\frac{\sum\nolimits_{i=1}^d \lambda_i^2}{d}} + t^2 + \frac{4}{d}\right) < 2{\rm e}^{-\frac{dt^2}{2}}.
$$
By defining
$$
\tau:=2t\sqrt{\frac{\sum_{i=1}^d \lambda_i^2}{d}} + t^2 + \frac{4}{d},
$$
we complete the proof.
\end{proof}

\section{Covering Arguments}
\label{app:covering-argument}

A convenient tool to discretize compact sets are nets.
In our proof, we will only need to discretize the unit Euclidean sphere in the definition of $\ell_2$-norm.
Let us recall a general definition of the $\varepsilon$-net.
\begin{definition}\citep{Vershynin2010Introduction}
An \emph{$\varepsilon$-net} in a totally bounded metric space $(X,d)$ is a finite subset $\mathcal N$
of $X$ such that for any $x\in X$ we have
\begin{equation*}
    \min_{z\in\mathcal N}d(x,z)<\varepsilon.
\end{equation*}
The \emph{metric entropy} of $X$ is a function $N(X,\varepsilon)$
defined as the minimum cardinality of an $\varepsilon$-net of $X$.
\end{definition}

The metric entropy of the Euclidean unit ball
can be easily bounded as follows.
\begin{lemma}(\citet{Foucart2017Mathematical}, Proposition C.3)\label{lemma:metric-entropy-bound}
Let $B_n$ be the unit ball in $\mathbb R^n$. Then
\begin{equation*}
    N(B_n, \varepsilon)\le\left(1+\frac{2}{\varepsilon}\right)^n.
\end{equation*}
\end{lemma}

$\varepsilon$- net allows us to evaluate the spectral norm of a square matrix ${\bm A}$
by only investigating a discrete set.
\begin{lemma}\citep{Vershynin2010Introduction}\label{lemma:covering-approximation}
Suppose $\mathcal N$ is a $\frac14$-net of $\bm S^{n-1}$.
Let ${\bm A}$ be an $n\times n$ matrix.
We have
    \begin{equation*}
        \max_{{\bm x}\in\bm{S}^{n-1}}\left|{\bm x}^{\rm T}{\bm A}{\bm x}\right|\le 2\max_{\x\in\mathcal N}\left|{\bm x}^{\rm T}{\bm A}{\bm x}\right|.
    \end{equation*}
\end{lemma}

\vskip 0.2in
\bibliography{refs}

\begin{thebibliography}{52}
\providecommand{\natexlab}[1]{#1}
\providecommand{\url}[1]{\texttt{#1}}
\expandafter\ifx\csname urlstyle\endcsname\relax
  \providecommand{\doi}[1]{doi: #1}\else
  \providecommand{\doi}{doi: \begingroup \urlstyle{rm}\Url}\fi

\bibitem[Absil et~al.(2004)Absil, Mahony, and Sepulchre]{Riemannian2004Absil}
P.-A. Absil, R.~Mahony, and R.~Sepulchre.
\newblock Riemannian geometry of grassmann manifolds with a view on algorithmic
  computation.
\newblock \emph{Acta Applicandae Mathematicae}, 80:\penalty0 199--220, 2004.

\bibitem[Absil et~al.(2006)Absil, Edelman, and Koev]{ABSIL2006}
P.-A. Absil, A.~Edelman, and P.~Koev.
\newblock On the largest principal angle between random subspaces.
\newblock \emph{Linear Algebra and its Applications}, 414\penalty0
  (1):\penalty0 288--294, 2006.

\bibitem[Baraniuk and Steeghs(2007)]{baraniuk2007compressive}
R.~Baraniuk and P.~Steeghs.
\newblock Compressive radar imaging.
\newblock In \emph{IEEE Radar Conference}, pages 128--133, April 2007.

\bibitem[Baraniuk et~al.(2008)Baraniuk, Davenport, Devore, and
  Wakin]{Baraniuk2015A}
R.~Baraniuk, M.~Davenport, R.~Devore, and M.~Wakin.
\newblock A simple proof of the restricted isometry property for random
  matrices.
\newblock \emph{Constructive Approximation}, 28\penalty0 (3):\penalty0
  253--263, 2008.

\bibitem[Bj\"{o}rck and Golub(1973)]{PrincipalAngles1973}
A.~Bj\"{o}rck and G.~H. Golub.
\newblock Numerical methods for computing the angles between linear subspaces.
\newblock \emph{Mathematics of Computation}, 27:\penalty0 579--594, 1973.

\bibitem[Cand{\`e}s(2006)]{Candes2006Compressive}
E.~J. Cand{\`e}s.
\newblock Compressive sampling.
\newblock In \emph{Proceedings of the international congress of
  mathematicians}, volume~3, pages 1433--1452. Madrid, Spain, 2006.

\bibitem[Cand{\`e}s(2008)]{Cand2008The}
J.~Cand{\`e}s, Emmanuel.
\newblock The restricted isometry property and its implications for compressed
  sensing.
\newblock \emph{Comptes Rendus Mathematique}, 346\penalty0 (9--10):\penalty0
  589--592, 2008.

\bibitem[Deutsch(1995)]{Deutsch1995The}
F.~Deutsch.
\newblock The angle between subspaces of a hilbert space.
\newblock In \emph{Approximation theory, wavelets and applications}, pages
  107--130. Springer, 1995.

\bibitem[Dirksen(2016)]{Dirksen2007Dimensionality}
S.~Dirksen.
\newblock Dimensionality reduction with subgaussian matrices: A unified theory.
\newblock \emph{Foundations of Computational Mathematics}, 16:\penalty0
  1367--1396, 2016.

\bibitem[Donoho(2006)]{donoho2006compressed}
D.~L. Donoho.
\newblock Compressed sensing.
\newblock \emph{IEEE Transactions on Information Theory}, 52\penalty0
  (4):\penalty0 1289--1306, 2006.

\bibitem[Drineas et~al.(2006)Drineas, Mahoney, and
  Muthukrishnan]{Drineas2006Sampling}
P.~Drineas, M.~W. Mahoney, and S.~Muthukrishnan.
\newblock Sampling algorithms for $l_2$ regression and applications.
\newblock In \emph{seventeenth annual ACM-SIAM symposium on Discrete
  algorithm}, 2006.

\bibitem[Dyer et~al.(2013)Dyer, Sankaranarayanan, and Baraniuk]{Dyer2013Greedy}
E.~L. Dyer, A.~C. Sankaranarayanan, and R.~G. Baraniuk.
\newblock Greedy feature selection for subspace clustering.
\newblock \emph{Journal of Machine Learning Research}, 14:\penalty0 2487--2517,
  2013.

\bibitem[Eftekhari and Wakin(2017)]{Eftekhari2017What}
A.~Eftekhari and M.~B. Wakin.
\newblock What happens to a manifold under a bi-lipschitz map?
\newblock \emph{Discrete $\&$ Computational Geometry}, 57\penalty0 (3), 2017.

\bibitem[Eftekhari et~al.(2015)Eftekhari, Yap, Rozell, and
  Wakin]{Eftekhari2015Therestricted}
A.~Eftekhari, H.~L. Yap, C.~J. Rozell, and M.~B. Wakin.
\newblock The restricted isometry property for random block diagonal matrices.
\newblock \emph{Applied and Computational Harmonic Analysis}, 38:\penalty0
  1--31, 2015.

\bibitem[Eldar and Mishali(2009)]{Eldar2009Robust}
Y.~C. Eldar and M.~Mishali.
\newblock Robust recovery of signals from a structured union of subspaces.
\newblock \emph{IEEE Transactions on Information Theory}, 55\penalty0
  (11):\penalty0 5302--5316, 2009.

\bibitem[Elhamifar and Vidal(2009)]{Elhamifar2009Sparse}
E.~Elhamifar and R.~Vidal.
\newblock Sparse subspace clustering.
\newblock In \emph{IEEE Conference on Computer Vision $\&$ Pattern
  Recognition}, 2009.

\bibitem[Elhamifar and Vidal(2013)]{Ehsan2013Sparse}
E.~Elhamifar and R.~Vidal.
\newblock Sparse subspace clustering: Algorithm, theory, and applications.
\newblock \emph{IEEE Transactions on Pattern Analysis and Machine Intelligence
  (PAMI)}, 35\penalty0 (11):\penalty0 2765--2781, 2013.

\bibitem[Foucart and Rauhut(2013)]{Foucart2017Mathematical}
S.~Foucart and H.~Rauhut.
\newblock \emph{A Mathematical Introduction to Compressive Sensing}.
\newblock Springer Science \& Business Media, 2013.

\bibitem[Frankl and Maehara(1990)]{Frankl1990}
P.~Frankl and H.~Maehara.
\newblock Some geometric applications of the beta distribution.
\newblock \emph{Annals of the Institute of Statistical Mathematics},
  42\penalty0 (3):\penalty0 463--474, Sept. 1990.

\bibitem[Gal\'{a}ntai and Heged\H{u}s(2006)]{PrincipalAngles2006}
A.~Gal\'{a}ntai and C.~J. Heged\H{u}s.
\newblock Jordan's principal angles in complex vector spaces.
\newblock \emph{Numerical Linear Algebra with Applications}, 13:\penalty0
  589--598, 2006.

\bibitem[Hamm and Lee(2008)]{Grassmann2008Hamm}
J.~Hamm and D.~D. Lee.
\newblock Grassmann discriminant analysis: A unifying view on subspace-based
  learning.
\newblock In \emph{Proceedings of International Conference on Machine Learning
  (ICML)}, pages 376--383, July 2008.

\bibitem[Heckel and B\"{o}lcskei(2015)]{Heckel2015robust}
R.~Heckel and H.~B\"{o}lcskei.
\newblock Robust subspace clustering via thresholding.
\newblock \emph{IEEE Transactions on Information Theory}, 61\penalty0
  (11):\penalty0 6320--6342, 2015.

\bibitem[Heckel et~al.(2017)Heckel, Tschannen, and
  B{\"o}lcskei]{Heckel2015dimensionality}
R.~Heckel, M.~Tschannen, and H.~B{\"o}lcskei.
\newblock Dimensionality-reduced subspace clustering.
\newblock \emph{Information and Inference: A Journal of the IMA}, 6\penalty0
  (3):\penalty0 246--283, 2017.

\bibitem[Jiao et~al.(2017)Jiao, Li, and Gu]{GlobalSIP2017}
Y.~Jiao, G.~Li, and Y.~Gu.
\newblock Principal angles preserving property of gaussian random projection
  for subspaces.
\newblock In \emph{IEEE Global Conference on Signal and Information Processing
  (GlobalSIP)}, pages 318--322, 2017.

\bibitem[Jiao et~al.(2018)Jiao, Shen, and Gu]{Jiao2018SubspaceDSW}
Y.~Jiao, X.~Shen, and Y.~Gu.
\newblock Subspace principal angle preserving property of gaussian random
  projection.
\newblock In \emph{IEEE Data Science Workshop}, pages 115--119, June 2018.

\bibitem[Johnson and Lindenstrauss(1984)]{Johnson1984Extensions}
W.~B. Johnson and J.~Lindenstrauss.
\newblock Extensions of lipschitz maps into a hilbert space.
\newblock \emph{Israel Journal of Mathematics}, 26\penalty0 (189):\penalty0
  189--206, 1984.

\bibitem[Jordan(1875)]{jordan1875essai}
C.~Jordan.
\newblock Essai sur la g{\'e}om{\'e}trie {\`a} $ n $ dimensions.
\newblock \emph{Bulletin de la Soci{\'e}t{\'e} math{\'e}matique de France},
  3:\penalty0 103--174, 1875.

\bibitem[Ledoux(2001)]{ledoux2001concentration}
M.~Ledoux.
\newblock \emph{The concentration of measure phenomenon}.
\newblock American Mathematical Soc., 2001.

\bibitem[Li and Gu(2018)]{LiGenTSP}
G.~Li and Y.~Gu.
\newblock Restricted isometry property of gaussian random projection for finite
  set of subspaces.
\newblock \emph{IEEE Transactions on Signal Processing}, 66\penalty0
  (7):\penalty0 1705--1720, 2018.

\bibitem[Li et~al.(2018)Li, Liu, and Gu]{optimal}
G.~Li, Q.~Liu, and Y.~Gu.
\newblock Rigorous restricted isometry property of low-dimensional subspaces.
\newblock \emph{arXiv:1801.10058}, 2018.

\bibitem[Lodhi and Bajwa(2018)]{Lodhi2017Detection}
M.~A. Lodhi and W.~U. Bajwa.
\newblock Detection theory for union of subspaces.
\newblock \emph{IEEE Transactions on Signal Processing}, 66\penalty0
  (24):\penalty0 6347--6362, Dec 2018.

\bibitem[Love and Heath(2005{\natexlab{a}})]{Limited2005Love}
D.~Love and R.~W. Heath.
\newblock Limited feedback unitary precoding for orthogonal space-time block
  codes.
\newblock \emph{IEEE Transactions on Signal Processing}, 53:\penalty0 64--73,
  2005{\natexlab{a}}.

\bibitem[Love and Heath(2005{\natexlab{b}})]{Limitedfeedback2005Love}
D.~J. Love and R.~W. Heath.
\newblock Limited feedback unitary precoding for spatial multiplexing systems.
\newblock \emph{IEEE Transactions on Information Theory}, 51:\penalty0
  2967--2976, 2005{\natexlab{b}}.

\bibitem[Mao and Gu(2014)]{Mao2014Compressed}
X.~Mao and Y.~Gu.
\newblock Compressed subspace clustering: A case study.
\newblock In \emph{2014 IEEE Global Conference on Signal and Information
  Processing (GlobalSIP)}, pages 453--457, 2014.

\bibitem[Mcwilliams and Montana(2014)]{Mcwilliams2014Subspace}
B.~Mcwilliams and G.~Montana.
\newblock Subspace clustering of high-dimensional data: a predictive approach.
\newblock \emph{Data Mining and Knowledge Discovery}, 28\penalty0 (3):\penalty0
  736--772, 2014.

\bibitem[Meng et~al.(2018)Meng, Li, Yan, and Gu]{Meng2018Ageneral}
L.~Meng, G.~Li, J.~Yan, and Y.~Gu.
\newblock A general framework for understanding compressed subspace clustering
  algorithms.
\newblock \emph{IEEE Journal of Selected Topics in Signal Processing},
  12\penalty0 (6):\penalty0 1504--1519, 2018.

\bibitem[Miao and Ben-Israel(1992)]{Miao1992Onprincipal}
J.~Miao and A.~Ben-Israel.
\newblock On principal angles between subspaces in $\mathbb{R}^n$.
\newblock \emph{Linear Algebra and its Applications}, 171:\penalty0 81--98,
  1992.

\bibitem[Miao and Ben-Israel(1996)]{Miao1996product}
J.~Miao and A.~Ben-Israel.
\newblock Product cosines of angles between subspaces.
\newblock \emph{Linear Algebra and its Applications}, 237:\penalty0 71--81,
  1996.

\bibitem[Paul et~al.(2013)Paul, Boutsidis, Magdon-Ismail, and
  Drineas]{Paul2013Random}
S.~Paul, C.~Boutsidis, M.~Magdon-Ismail, and P.~Drineas.
\newblock Random projections for support vector machines.
\newblock In \emph{Proceedings of the Sixteenth International Conference on
  Artificial Intelligence and Statistics}, pages 498--506, April 2013.

\bibitem[Shen et~al.(2018)Shen, Jiao, and Gu]{Shen2018Subspace}
X.~Shen, Y.~Jiao, and Y.~Gu.
\newblock Subspace data visualization with dissimilarity based on principal
  angle.
\newblock In \emph{IEEE Data Science Workshop}, pages 16--20, 2018.

\bibitem[Shi et~al.(2012)Shi, Shen, Hill, and Hengel]{Shi2012Is}
Q.~Shi, C.~Shen, R.~Hill, and A.~Hengel.
\newblock Is margin preserved after random projection?
\newblock In \emph{Proceedings of International Conference on Machine Learning
  (ICML)}, 2012.

\bibitem[Soltanolkotabi and Candes(2012)]{soltanolkotabi2012geometric}
M.~Soltanolkotabi and E.~J. Candes.
\newblock A geometric analysis of subspace clustering with outliers.
\newblock \emph{The Annals of Statistics}, 40\penalty0 (4):\penalty0
  2195--2238, 2012.

\bibitem[Vershynin(2010)]{Vershynin2010Introduction}
R.~Vershynin.
\newblock Introduction to the non-asymptotic analysis of random matrices.
\newblock \emph{arxiv:1011.3027}, 2010.

\bibitem[Wang et~al.(2015)Wang, Wang, and Singh]{Wang2019Adeterministic}
Y.~Wang, Y.-X. Wang, and A.~Singh.
\newblock A deterministic analysis of noisy sparse subspace clustering for
  dimensionality-reduced data.
\newblock In \emph{International Conference on Machine Learning}, pages
  1422--1431, 2015.

\bibitem[Wang and Xu(2016)]{Noisy2016Wang}
Y.-X. Wang and H.~Xu.
\newblock Noisy sparse subspace clustering.
\newblock \emph{Journal of Machine Learning Research}, 17\penalty0
  (12):\penalty0 1--41, 2016.

\bibitem[Wong(1967)]{Wong1967Differential}
Y.~C. Wong.
\newblock Differential geometry of grassmann manifolds.
\newblock \emph{Proceedings of the National Academy of Sciences of the United
  States of America}, 57\penalty0 (3):\penalty0 589, 1967.

\bibitem[Xu et~al.(2019)Xu, Li, and Gu]{xingyu}
X.~Xu, G.~Li, and Y.~Gu.
\newblock Johnson-lindenstrauss property implies subspace restricted isometry
  property.
\newblock \emph{arXiv:1905.09608}, 2019.

\bibitem[Ye and Lim(2016)]{YeLim}
K.~Ye and L.~Lim.
\newblock Schubert varieties and distances between subspaces of different
  dimensions.
\newblock \emph{SIAM Journal on Matrix Analysis and Applications}, 37\penalty0
  (3):\penalty0 1176--1197, 2016.

\bibitem[You et~al.(2016)You, Robinson, and Vidal]{You2016Scalable}
C.~You, D.~P. Robinson, and R.~Vidal.
\newblock Scalable sparse subspace clustering by orthogonal matching pursuit.
\newblock In \emph{IEEE Conference on Computer Vision and Pattern Recognition
  (CVPR)}, pages 3918--3927, June 2016.

\bibitem[Yu et~al.(2014)Yu, Wang, and Samworth]{Yu2014auseful}
Y.~Yu, T.~Wang, and R.~J. Samworth.
\newblock A useful variant of the davis--kahan theorem for statisticians.
\newblock \emph{Biometrika}, 102\penalty0 (2):\penalty0 315--323, 2014.

\bibitem[Zhai et~al.(2017)Zhai, Zhang, Zhang, Li, and Plaza]{Zhai2017Anew}
H.~Zhai, H.~Zhang, L.~Zhang, P.~Li, and A.~Plaza.
\newblock A new sparse subspace clustering algorithm for hyperspectral remote
  sensing imagery.
\newblock \emph{IEEE Geoscience and Remote Sensing Letters}, 14\penalty0
  (1):\penalty0 43--47, 2017.

\bibitem[Zhu et~al.(2014)Zhu, Huang, Torre, and Lucey]{Complex2014Zhu}
Y.~Zhu, D.~Huang, F.~D.~L. Torre, and S.~Lucey.
\newblock Complex non-rigid motion 3d reconstruction by union of subspaces.
\newblock In \emph{IEEE Conference on Computer Vision and Pattern Recognition
  (CVPR)}, pages 1542--1549, June 2014.

\end{thebibliography}

\end{document}